\documentclass[11pt]{article}
\pdfoutput=1

\usepackage{enumerate}
\usepackage[OT1]{fontenc}
\usepackage{amsmath,amssymb}
\usepackage{subfigure}
\usepackage{natbib}
\usepackage[usenames]{color}
\usepackage{multirow}
\usepackage[colorlinks,
            linkcolor=red,
            anchorcolor=blue,
            citecolor=blue
            ]{hyperref}
            
\usepackage{enumitem}
\usepackage{mathrsfs}
\usepackage{fullpage}
\usepackage[protrusion=false,expansion=true]{microtype}
\usepackage{mylatexstyle}

\begin{document}

\title{\Huge Risk Bounds of Multi-Pass SGD for Least Squares  in  the Interpolation Regime}



\author
{
Difan Zou\thanks{Equal Contribution} \thanks{Department of Computer Science, University of California, Los Angeles, CA 90095, USA; e-mail: {\tt knowzou@cs.ucla.edu}} 
	~~
	Jingfeng Wu$^*$\thanks{Department of Computer Science, Johns Hopkins University, Baltimore, MD 21218, USA; e-mail: {\tt uuujf@jhu.edu}} 
	~~
	Vladimir Braverman\thanks{Department of Computer Science, Johns Hopkins University, Baltimore, MD 21218, USA; e-mail: {\tt
vova@cs.jhu.edu}}
	~~
	Quanquan Gu\thanks{Department of Computer Science, University of California, Los Angeles, CA 90095, USA; e-mail: {\tt qgu@cs.ucla.edu}}
	~~
	Sham M. Kakade\thanks{Department of Computer Science \& Statistics, Harvard University, Cambridge, MA 02138, USA; e-mail: {\tt sham@seas.harvard.edu}}
}

\date{}
\maketitle

\begin{abstract}%
Stochastic gradient descent (SGD) has achieved great success due to its superior performance in both optimization and generalization. Most of existing generalization analyses are made for single-pass SGD, which is a less practical variant compared to the commonly-used multi-pass SGD. Besides, theoretical analyses for multi-pass SGD often concern a worst-case instance in a class of problems, which may be pessimistic to explain the superior generalization ability for some particular problem instance. The goal of this paper is to sharply characterize the generalization of multi-pass SGD, by developing an instance-dependent excess risk bound for least squares in the interpolation regime, which is expressed as a function of the iteration number, stepsize, and data covariance. We show that the excess risk of SGD can be exactly decomposed into the excess risk of GD and a positive fluctuation error, suggesting that SGD always performs worse, instance-wisely, than GD, in generalization. On the other hand, we show that although SGD needs more iterations than GD to achieve the same level of excess risk, it saves the number of stochastic gradient evaluations, and therefore is preferable in terms of computational time.
\end{abstract}

\section{Introduction}
\emph{Stochastic gradient descent} (SGD) is one of the workhorses in modern machine learning  due to its efficiency and scalability in training and good ability in generalization to unseen test data.
From the optimization perspective, the efficiency of SGD is well understood.
For example, to achieve the same level of optimization error, SGD saves the number of gradient computation compared to its deterministic counterpart, i.e., batched \emph{gradient descent} (GD) \citep{bottou2007tradeoffs,bottou2018optimization}, and therefore saves the total amount of running time.
However, the generalization ability (e.g., excess risk bounds) of SGD is far less clear, especially from theoretical perspective.

\emph{Single-pass} SGD, a less practical SGD variant where each training data is used only once, has been extensively studied in theory.
In particular, a series of works establishes tight excess risk bounds of single-pass SGD in the setting of learning least squares \citep{bach2013non,dieuleveut2017harder,jain2017markov,jain2017parallelizing,neu2018iterate,ge2019step,zou2021benign,wu2021last}.
In practice, though, one often runs SGD with \emph{multiple passes} over the training data and outputs the final iterate, which is referred to as \emph{multi-pass SGD} (or simply SGD in the rest of this paper when there is no confusion).
Compared to single-pass SGD that has limited number of optimization steps, multi-pass SGD allows the algorithm to perform arbitrary number of optimization steps, which is more powerful in optimizing the empirical risk and thus leads to smaller bias error \citep{pillaud2018statistical}.

Despite the extensive application of multi-pass SGD in practice, there are only a few theoretical techniques being developed to study the generalization of multi-pass SGD.
One among them is through the method of \emph{uniform stability} \citep{elisseeff2005stability,hardt2016train},
which is defined as the change of the model outputs when applying a small change in the training dataset. However, the stability based generalization bound is a worst-case guarantee, which is relatively crude and does not show difference between GD and SGD (See, e.g., \citet{chen2018stability} showed GD and SGD have the same stability parameter in the convex smooth setting). 
On the contrary, one easily observes a generalization difference between SGD and GD even in learning the simplest least square problem (see Figure \ref{fig:risk-comparison}). 
In addition, \citet{lin2017optimal,pillaud2018statistical,mucke2019beating} explored the risk bounds for multi-pass SGD using the \emph{operator methods} that are originally developed for analyzing single-pass SGD.
Their bounds are sharp in the minimax sense for a class of least square problems that satisfy certain \emph{source condition} (which restricts the norm of the optimal parameter) and \emph{capacity condition} (or effective dimension, which restricts the spectrum of the data covariance matrix). 
Still, their bounds are not problem-dependent, and could be pessimistic for benign least square instances.

In this paper, our goal is to establish sharp algorithm-dependent and problem-dependent excess risk bounds of multi-pass SGD for least squares.
Our focus is the \emph{interpolation regime} where the training data can be perfectly fitted by a linear interpolator (which holds almost surely when the number of parameter $d$ exceeds the number of training data $n$). 
We assume the data has a sub-Gaussian tail \citep{bartlett2020benign}.
Our main contributions are summarized as follows:
\begin{itemize}[leftmargin=*]
    \item We show that for any iteration number and stepsize, the excess risk of SGD can be exactly decomposed into the excess risk of GD (with the same stepsize and iteration number) and the so-called \textit{fluctuation error}, which is attributed to the accumulative variance of stochastic gradients in all iterations. This suggests that GD (with optimally tuned hyperparameters) always achieves smaller excess risk than SGD for least square problems.   
    \item We further establish problem-dependent bounds for the excess risk of GD and the fluctuation error, stated as a function of the eigenspectrum of the data covariance, iteration number, training sample size, and stepsize. Compared to the bounds proved in prior works \citep{lin2017optimal,pillaud2018statistical,mucke2019beating}, our bounds hold under a milder assumption on the data covariance and ground-truth model. Moreover, our bounds can be applied to a wider range of iteration numbers $t$, i.e., for any $t>0$, in contrast to the prior results that will explode when $t\rightarrow \infty$.

    \item We develop a new suite of proof techniques for analyzing the excess risk of multi-pass SGD. Particularly, the key to our analysis is describing the error covariance based on the tensor operators defined by the second-order and fourth-order moments of the empirical data distribution (i.e., sampling with replacement from the training dataset), rather than the operators used in the single-pass SGD analysis that are defined based on the (population) data distribution \citep{jain2017markov,zou2021benign} (i.e., sampling from the data distribution), together with a sharp characterization on the properties of the operators. 
\end{itemize}
Our developed excess risk bounds for SGD and GD have  important implications on the complexity comparison between GD and SGD: to achieve the same order of excess risk, while SGD may need more iterations than GD, it can have fewer stochastic gradient evaluations than GD. For example, consider the case that the data covariance matrix has a polynomially decaying spectrum with rate $i^{-(1+r)}$, where $r>0$ is an absolute constant. In order to achieve the same order of excess risk, we have the following comparison in terms of iteration complexity and gradient complexity\footnote{We define the gradient complexity as the number of required stochastic gradient evaluations to achieve a target excess risk, which is closely related to the total computation time.}: 
\begin{itemize}[leftmargin=*]
    \item \textit{Iteration Complexity:} SGD needs to take $\tilde \cO(n^{\max\{0.5, \frac{r}{r+1}\}})$ more iterations than GD, with optimally tuned iteration number and stepsize.
    \item \textit{Gradient Complexity:} SGD needs $\tilde \cO(n^{\max\{0.5, \frac{1}{r+1}\}})$ less stochastic gradient evaluations than GD.
\end{itemize}

\begin{figure}
    \centering
    \subfigure[\small{$\lambda_i=i^{-1}\log^{-2}(i+10)$}]{\includegraphics[width=0.45\textwidth]{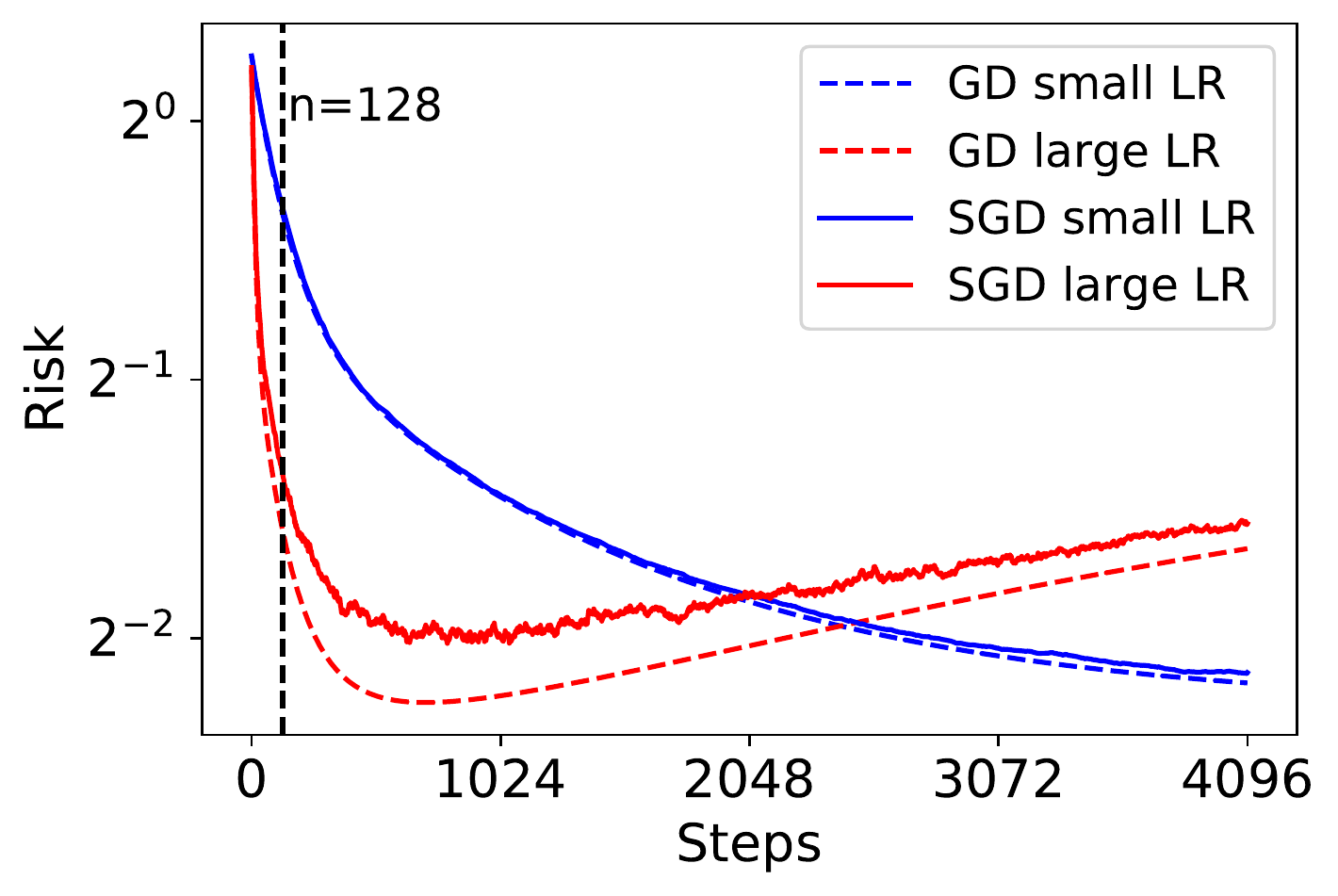}}
      \subfigure[\small{$\lambda_i=i^{-2}$}]{\includegraphics[width=0.45\textwidth]{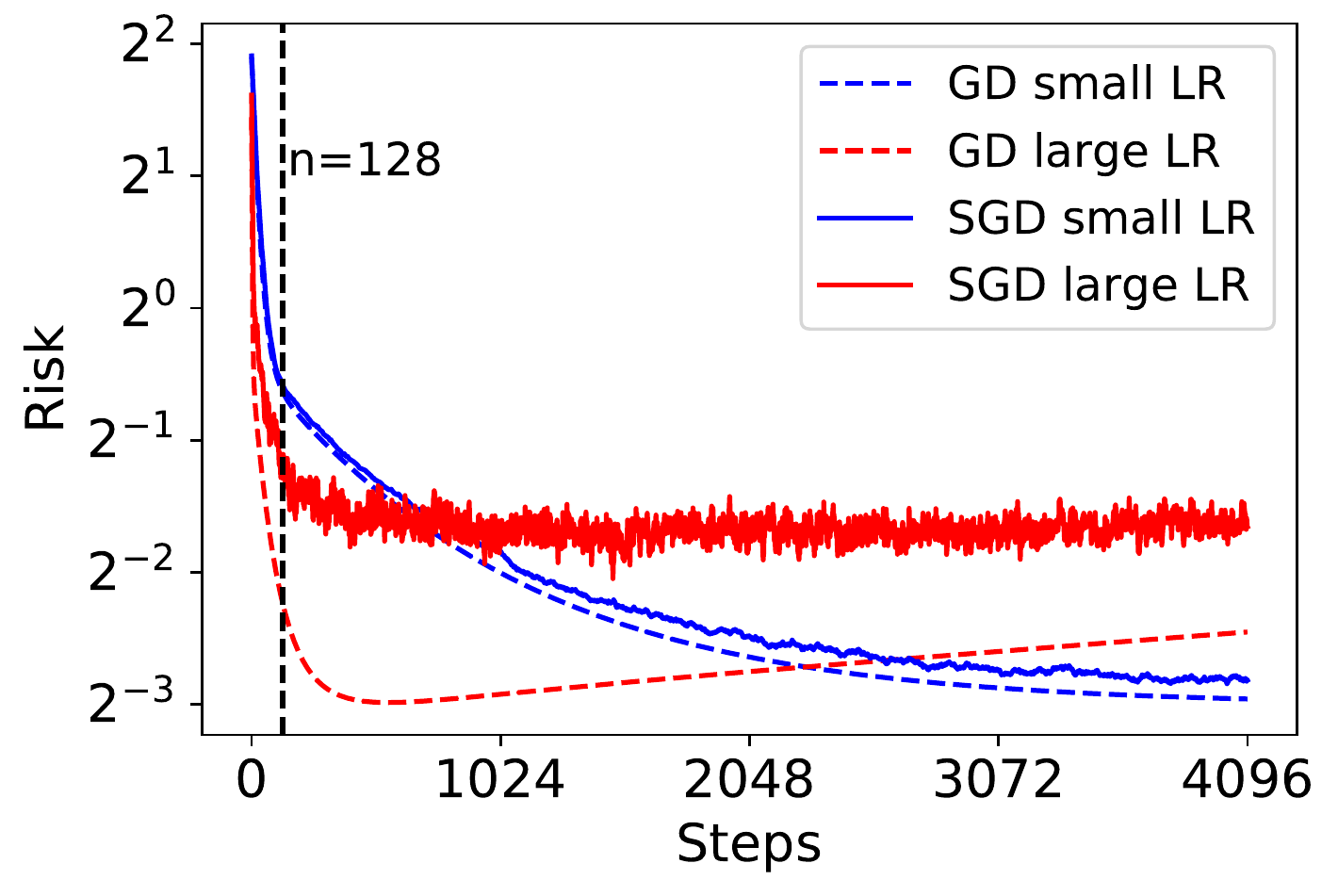}}
    \caption{\small 
      Excess risk comparison between SGD and GD with large and small stepsizes. 
      The true parameter $\wb^*$ is randomly drawn from $\cN(0, \Ib)$ and the model noise variance $\sigma^2=1$.
      The problem dimension is $d=256$, and we randomly draw $n = 128$ training data.
      We consider two data covariance  with eigenspectrum $\lambda_i=i^{-1}\log^{-2}(i+10)$ and $\lambda_i=i^{-2}$.
      For SGD, the reported risk is averaged over $100$ repeats of the algorithm's randomness.
      The large stepsize is $\eta = 0.2$ and the small stepsize is $\eta = 0.02$.\label{fig:risk-comparison}}
    \vspace{-.5cm}
\end{figure}

\paragraph{Notations.}
For a scalar $n>0$. we use $\poly(n)$ to define some positive high-degree polynomial functions of $n$. For two positive-value functions $f(x)$ and $g(x)$ we write $f(x)\lesssim g(x)$ if $f(x) \le c g(x)$ for some constant 
$c>0$, we write $f(x) \gtrsim g(x)$ if $g(x) \lesssim f(x)$, and $f(x) \eqsim g(x)$ if both $f(x)\lesssim g(x)$ and $g(x) \lesssim f(x)$ hold.
We use $\tilde \cO(\cdot)$ to hide some polylogarithmic factors in the standard big-$\cO$ notation.

\section{Related Work}

\paragraph{Optimization.}
Regarding optimization efficiency, the benefits of SGD is well understood \citep{bottou2007tradeoffs,bottou2018optimization,ma2018power,bassily2018exponential,vaswani2019fast,vaswani2019painless}.
For example, for strongly convex losses (can be relaxed with certain growth conditions), GD has less iteration complexity, but SGD enjoys less gradient complexity \citep{bottou2007tradeoffs,bottou2018optimization}.
More recently, it is shown that SGD can converge at an exponential rate in the interpolating regime \citep{ma2018power,bassily2018exponential,vaswani2019fast,vaswani2019painless}, therefore SGD can match the iteration complexity of GD.
Nevertheless, all the above results are regrading the optimization performance;
our focus in this paper is to study the generalization performance of SGD (and GD).


\paragraph{Risk Bounds for Multi-Pass SGD.}
The risk bounds of multi-pass SGD are also studied from the operator perspective \citep{rosasco2015learning,lin2017optimal,pillaud2018statistical,mucke2019beating}.
The work by \citet{rosasco2015learning} focused on \emph{cyclic SGD}, i.e., SGD with multiple passes but fixed sequence on the training data.
Their results are limited to small stepsizes ($\gamma = \cO(1/n)$), while ours allow constant stepsize.
Similar to \citet{lin2017optimal,pillaud2018statistical,mucke2019beating}, we decompose the population risk of SGD iterates into a risk term caused by batch GD iterates and a fluctuation error term between SGD and GD iterates. But our methods of bounding the fluctuation error are different (see more in..). 
Moreover, our results are based on different assumptions: \citet{lin2017optimal,pillaud2018statistical,mucke2019beating} assumed strong finiteness on the optimal parameter, and their results only apply to data covariance with a specific type of spectrum (nearly polynomially decaying ones); in contrast, our results assume a Gaussian prior on the optimal parameter (which might not admit a finite norm), and our results cover more general data covariance (inlcuding those with polynomially decaying spectrum). \citet{lei2021generalization} studied risk bounds for multi-pass SGD with general convex loss. When applied to least square problems, their bounds are cruder than ours.


\paragraph{Uniform Stability.}
Another approach for characterizing the generalization of multi-pass SGD is through \emph{uniform stability} \citep{hardt2016train,chen2018stability,kuzborskij2018data,zhang2021stability,bassily2020stability}.
There are mainly two differences between this and our approach. 
First, we directly bound the excess risk of SGD; but the uniform stability can only bound the generalization error, there needs an additional triangle inequality to relate excess risk with generalization error plus optimization error (plus approximation error) --- this inequality can easily be loose (consider the algorithmic regularization effects).
Secondly, the uniform stability bound is also crude. For example, in the non-strongly convex setting, the uniform stability bound for SGD/GD linearly scales with the total optimization length (i.e., sum of stepsizes), which grows as $t$ \citep{hardt2016train,chen2018stability,kuzborskij2018data,zhang2021stability,bassily2020stability} (this is minimaxly unavoidable according to \citet{zhang2021stability,bassily2020stability}). 
Notably, \citet{bassily2020stability} extended the uniform stability approach to the non-convex and smooth setting. We left such an extension of our method as a future work.

\section{Problem Setup}
Let $\xb$ be a feature vector in a Hilbert space $\cH$ (its dimension is denoted by $d$, which is possibly infinite) and $y \in \RR$ be its response, and assume that they jointly follow an unknown population distribution $\cD$.
In linear regression problems, the \emph{population risk} of a parameter $\wb$ is defined by
\[
 L_{\cD}(\wb) := \frac{1}{2}\EE_{(\xb, y)\sim D} (\la \xb, \wb \ra - y)^2,
\]
and the \emph{excess risk} is defined by
\begin{equation}\label{eq:excess_risk}
\risk(\wb) := L_D(\wb) - \min_{\wb} L_D(\wb) = \frac{1}{2}  \|\wb-\wb^*\|_{\Hb}^2, \quad \ \text{where}\ \Hb := \EE_{\cD} [\xb \xb^\top] .
\end{equation}
In the statistical learning setting, the population distribution $\cD$ is unknown, and one is provided with a set of $N$ training samples, $\cS = (\xb_i, y_i)_{i=1}^n$, that are drawn independently at random from the population distribution.
We also use $\Xb := (\xb_1, \dots, \xb_n)^\top$ and $\yb := (y_1, \dots, y_n)^\top$ to denote the concatenated features and labels, respectively. 
The linear regression problems aim to find a parameter based on the training set $\cS$ that affords a small excess risk.

\paragraph{Multi-Pass SGD.}
We are interested in solving the linear regression problem using multi-pass \emph{stochastic gradient descent} (SGD).
The algorithm generates a sequence of iterates $(\wb_t)_{t\ge 1}$ according to the following update rule:
the initial iterate is $\wb_0 = \boldsymbol{0}$ (which can be assumed without lose of generality); then at each iteration, an example $(\xb_{i_t}, y_{i_t})$ is drawn from $\cS$ uniformly at random, and the iterate is updated by
\begin{align*}
\wb_{t+1} = \wb_t - \eta\cdot \xb_{i_t} (\xb_{i_t}^\top\wb_t - y_{i_t}),
\end{align*}
where $\eta > 0$ is a constant stepsize (i.e., learning rate).

\paragraph{GD.}
Another popular algorithm is \emph{gradient descent} (GD). 
For the clarify of notations, we use $(\hat{\wb}_t)_{t \ge 1}$ to denote the GD iterates,
which follow the following updates:
\begin{align*}
\hat{\wb}_{t+1} = \hat{\wb}_t - \eta\cdot \frac{1}{n} \sum_{i=1}^n \xb_i (\xb_i^\top \hat{\wb}_t - y_i),\quad \hat\wb_0=\boldsymbol{0},
\end{align*}
where $\eta > 0$ is a constant stepsize.

\paragraph{Notations and Assumptions.}
We use $\Hb := \EE[\xb\otimes\xb]$ to denote the population data covariance matrix.
The eigenvalues of $\Hb$ is denoted by $(\lambda_i)_{i \ge 1}$, sorted in non-increasing order.
Given the training data $(\Xb, \yb)$, we define $\bepsilon=\yb-\Xb\wb^*$ the collection of model noise, $\Ab = \Xb\Xb^\top$ as the gram matrix, and $\bSigma = n^{-1}\Xb^\top\Xb$ as the empirical covariance. Then the \textit{minimum-norm solution} is defined by
\begin{align*}
\hat\wb : = (\Xb^\top\Xb)^{\dagger}\Xb^\top\yb = \Xb^\top\Ab^{-1}\yb.
\end{align*}
It is clear that with appropriate stepsizes, both SGD and GD algorithms converge to $\hat{\wb}$ \citep{gunasekar2018characterizing,bartlett2020benign}.

The assumptions required by our theorems are summarized in below.
\begin{assumption}\label{assump:data_distribution}
For the linear regression problem:
\begin{enumerate}[label=\Alph*]
    \item  The components of $\Hb^{-1/2}\xb$ are independent and $1$-subGaussian.\label{assump:item:subgaussian_data}
    \item The response $y$ is generated by $y := \la\wb^*, \xb\ra + \xi$, where  $\wb^*$ is the ground truth weight vector and $\xi$ is a noise independent of $\xb$. Furthermore, the additive noise satisfies $\EE[\xi]=0$, $\EE[\xi^2]=\sigma^2$.\label{assump:item:well_specified_noise}
    \item The ground truth $\wb^*$ follows a Gaussian prior $\cN(\boldsymbol{0},\ \omega^2\cdot\Ib)$, where $\omega^2$ is a constant. \label{assump:item:gaussian_prior}
    \item The minimum-norm solution $\hat{\wb}$ linearly interpolates all training data, i.e., \(y_i = \hat{\wb}^\top \xb_i \)  for every \( i=1,\dots,n\). \label{assump:item:interpolator}
\end{enumerate}
\end{assumption}
Assumptions \ref{assump:data_distribution}\ref{assump:item:subgaussian_data} and \ref{assump:item:well_specified_noise} are standard for analyzing overparameterized linear regression problem \citep{bartlett2020benign,tsigler2020benign}.
Assumption \ref{assump:data_distribution}\ref{assump:item:gaussian_prior} is also widely adopted in analyzing least square problems (see, e.g.,  \citet{ali2019continuous,dobriban2018high,xu2019number}).
Finally, Assumption \ref{assump:data_distribution}\ref{assump:item:interpolator} holds almost surely when $d > n$, i.e., the number of parameter exceeds the number of data.


In the following, the presented risk bounds will hold (i) with high-probability with respect to the randomness of sampling feature vectors $\Xb$, and (ii) in expectation with respect to the randomness of multi-pass SGD algorithm, the randomness of sampling additive noise  $\bepsilon$  and the randomness of the true parameter $\wb^*$ as a prior.
For these purpose, we will use $\EE_{\sgd}, \EE_{\wb^*}$ to refer to taking expectation with respect to the SGD algorithm and the prior distribution of $\wb^*$, respectively.

\section{Main Results}


Our first theorem shows that, under the same stepsize and number of iterates, SGD always generalizes worse than GD.

\begin{theorem}[Risk decomposition]\label{thm:risk_decomposition}
Suppose that Assumption \ref{assump:data_distribution}\ref{assump:item:interpolator} holds.
Then the excess risk of SGD can be decomposed by
\[
\EE_{\sgd} \big[\risk(\wb_t)\big] = \risk(\hat{\wb}_t) + \FlucError(\wb_t).
\]
Moreover, the fluctuation error is always positive.
\end{theorem}

\paragraph{A Risk Comparison.}
Theorem \ref{thm:risk_decomposition} shows that, in the interpolation regime, SGD affords a strictly large excess risk than GD, given the same hyperparameters (stepsize $\eta$ and number of iterates $t$). 
Therefore, despite of a possibly higher computational cost, the optimally tuned GD \emph{dominates} the optimally tuned SGD in terms of the generalization performance.
This observation is verified empirically by experiments in Figure~\ref{fig:risk-comparison}.


Theorem \ref{thm:risk_decomposition} relates the risk of SGD iterates to that of GD iterates. This idea has appeared in earlier literature \citep{lin2017optimal,pillaud2018statistical,mucke2019beating}. 

Our next theorem is to characterize the fluctuation error of SGD (with respect to GD).
\begin{theorem}[Fluctuation error bound]\label{thm:upperbound_fluctuation}
Suppose that Assumptions \ref{assump:data_distribution}\ref{assump:item:subgaussian_data}, \ref{assump:item:well_specified_noise} and \ref{assump:item:interpolator} all hold.
Then for every $n\ge 1$, $t\ge 1$ and $\eta\le c/\tr(\Hb)$ for some absolute constant $c$, with probability at least $1-1/\poly(n)$, it holds that
\begin{align*}
    & \FlucError (\wb_t) 
    \lesssim \notag\\
    & \bigg[\log(t)\cdot\bigg(\frac{\tr(\Hb)\log(n)}{t}+\frac{k^\dagger\log^{5/2}(n)}{n^{1/2}t}\bigg)+\frac{\log^{5/2}(n)\eta}{n^{1/2}}\cdot\sum_{i>k^\dagger}\lambda_i\bigg)\bigg]
\cdot\min\big\{\|\hat\wb\|_2^2,\ t\eta\cdot \|\hat{\wb} \|_{\bSigma}^2\big\},
\end{align*}
where $k^\dagger \ge 0$ is an arbitrary index (can be infinity).
\end{theorem}

We first explain the factor $\min\big\{\|\hat\wb\|_2^2,\ t\eta\cdot \| \hat{\wb} \|_{\bSigma}^2\big\}$ in our bound. 
First of all, when the interpolator $\hat\wb$ has a small $\ell_2$-norm, the quantity is automatically small.
Furthermore, $\| \hat{\wb} \|_{\bSigma}^2 \lesssim \omega^2 \lesssim 1$ easily holds under mild assumptions on $\wb^*$, e.g., Assumption \ref{assump:data_distribution}\ref{assump:item:gaussian_prior}.
Then, for finite $t$ one can bound the factor with $\min\big\{\|\hat\wb\|_2^2,\ t\eta\cdot \| \hat{\wb} \|_{\bSigma}^2\big\} \lesssim \omega^2 \eta t$.


More interestingly, for SGD with constant stepsize ($\gamma \eqsim 1/\tr(\Hb)$) and infinite optimization steps ($t \to \infty$), our risk bound can still vanish, while all risk bounds in prior works \citep{lin2017optimal,pillaud2018statistical,mucke2019beating} are vacuous.
To see this, one can first set $k^\dagger = \infty$ and $t \to \infty$ in Theorem \ref{thm:upperbound_fluctuation}, so the fluctuation error vanishes. 
Secondly, note that GD with constant stepsize converges to the minimum-norm interpolator $\hat{\wb}$, so the risk of GD converges to the risk of $\hat{\wb}$, which is known to vanish for data covariance that enables ``benign overfitting''  \citep{bartlett2020benign}.
Combining these with Theorm \ref{thm:risk_decomposition} gives a generalization bound of SGD with constant stepsize and infinite optimization steps. 

To complement the above results, we provide the following finite-time risk bound for GD. Nonetheless, we emphasize that any risk bound for GD can be plugged into Theorems \ref{thm:risk_decomposition} and~\ref{thm:upperbound_fluctuation} to obtain a risk bound for SGD.
\begin{theorem}[GD risk]\label{thm:GD_earlystop}
Suppose that Assumptions \ref{assump:data_distribution}\ref{assump:item:subgaussian_data}, \ref{assump:item:well_specified_noise} and \ref{assump:item:gaussian_prior} all hold.
Then for every $n\ge 1$, $t \ge 1$ and $\eta < {1} / {\norm{\Hb}_2}$, with probability at least $1-1/\poly(n)$, it holds that
\begin{equation*}
\EE_{\wb^*, \bepsilon} [ \risk( \hat{\wb}_t) ] 
\lesssim \omega^2 \cdot
     \Bigg(  \frac{\tilde{\lambda}^2 }{n^2 }  \cdot \sum_{i\le k^*} \frac{1}{\lambda_i} + \sum_{i>k^*} \lambda_i \Bigg)
    + \sigma^2\cdot \rbr{\frac{k^*}{n} + \frac{n}{\tilde{\lambda}^2} \sum_{i>k^*}\lambda_i^2 },
\end{equation*}
where 
$k^* := \min\{k: n\lambda_{k+1} \le \frac{n}{\eta t} + \sum_{i > k} \lambda_i \}$ and 
$\tilde{\lambda} := \frac{n}{\eta t} + \sum_{i > k^*} \lambda_i$.
\end{theorem}

The bound presented in Theorem \ref{thm:GD_earlystop} is comparable to that for ridge regression established by \citet{tsigler2020benign} and will be much better than the bound of single-pass SGD when the signal-to-noise ratio is large \citep[Theorem 5.1]{zou2021benefits}, e.g., $\omega^2 \gg \sigma^2$.
In fact, Theorem \ref{thm:GD_earlystop} is proved via a reduction to ridge regression results (see Section \ref{sec:gd_risk} for more details).
In particular, the quantity $n/(\eta t)$ for GD is an analogy to the regularization parameter $\lambda$ for ridge regression \citep{yao2007early,raskutti2014early,wei2017early,ali2019continuous}.
As a final remark, the assumption that $\wb^*$ follows a Gaussian prior is the main concealing in Theorem \ref{thm:GD_earlystop} (which is not required by \citet{tsigler2020benign} for ridge regression).
The Gaussian prior on $\wb^*$ is known to allow a connection between early stopped GD with ridge regression \citep{ali2019continuous}. 
We conjecture that this assumption is not necessary and potentially removable.

Combining Theorems \ref{thm:risk_decomposition}, \ref{thm:upperbound_fluctuation}  and \ref{thm:GD_earlystop}, we obtain the following excess risk bound for multi-pass SGD in the interpolating least square problems:
\begin{corollary}\label{coro:expected_SGD_error}
Suppose that Assumptions \ref{assump:data_distribution}\ref{assump:item:subgaussian_data}, \ref{assump:item:well_specified_noise}, \ref{assump:item:gaussian_prior} and \ref{assump:item:interpolator} all hold.
Then with probability at least $1-1/\poly(n)$, it holds that
\begin{align*}
&\EE_{\mathrm{SGD}, \wb^*,\bepsilon}\big[\cE(\wb_t)\big]
\lesssim \omega^2 \cdot
     \Bigg(  \frac{\tilde{\lambda}^2 }{n^2 }  \cdot \sum_{i\le k^*} \frac{1}{\lambda_i} + \sum_{i>k^*} \lambda_i \Bigg)
    + \sigma^2\cdot \rbr{\frac{k^*}{n} + \frac{n}{\tilde{\lambda}^2} \sum_{i>k^*}\lambda_i^2 }\notag\\
&\quad +\big[\omega^2\tr(\Hb) +\sigma^2)\big]\eta\cdot \bigg[\log(t)\cdot\bigg(\tr(\Hb)\log(n)+\frac{k^*\log^{5/2}(n)}{n^{1/2}}\bigg)+\frac{\log^{5/2}(n)t\eta}{n^{1/2}}\cdot\sum_{i>k^*}\lambda_i\bigg)\bigg],
\end{align*}
where 
$k^* := \min\{k: n\lambda_{k+1} \le \frac{n}{\eta t} + \sum_{i > k} \lambda_i \}$ and 
$\tilde{\lambda} := \frac{n}{\eta t} + \sum_{i > k^*} \lambda_i$.
\end{corollary}

\paragraph{Comparison with Existing Results.}
We now discuss differences and connections between our bound and existing ones for multi-pass SGD \citep{lin2017optimal,pillaud2018statistical,mucke2019beating}. 
First, we highlight that our bound is \emph{problem-dependent} in the sense that the bound is stated as a function of the spectrum of data covariance; 
in contrast, existing papers only provide a minimax analysis for multi-pass SGD.
Secondly, we rely on a different set of assumptions from the aforementioned papers. 
In particular, \citet{pillaud2018statistical} requires a \emph{source condition} on the data covariance, and \citet{lin2017optimal,mucke2019beating} requires an \emph{effective dimension} (defined by the data covariance) to be small, but our results are more general regarding the data covariance. 
Moreover, we assume $\wb^*$ follows a Gaussian prior (Assumption \ref{assump:data_distribution}\ref{assump:item:gaussian_prior}), but existing works require a \emph{source condition} on $\wb^*$, which are not directly comparable. 



The following corollary characterizes the risk of multi-pass SGD for data covariance with polynomially decaying spectrum.

\begin{corollary}\label{coro:SGD_error_polydecay}
Suppose that Assumptions \ref{assump:data_distribution}\ref{assump:item:subgaussian_data}, \ref{assump:item:well_specified_noise}, \ref{assump:item:gaussian_prior} and \ref{assump:item:interpolator} all hold.
Assume the spectrum of $\Hb$ decays polynomially, i.e., $\lambda_i = i^{-1-r}$ for some absolute constant $r>0$, then with probability at least $1-1/\poly(n)$, it holds that
\begin{align*}
\EE_{\wb^*,\bepsilon}[\cE(\hat\wb_t)]&\lesssim   \omega^2 \cdot
    (t\eta)^{-r/(r+1)}
    + \sigma^2\cdot \frac{(t\eta)^{1/(r+1)}}{n},\notag\\
\EE_{\sgd, \wb^*,\bepsilon}[\cE(\wb_t)]
&\lesssim\omega^2 \cdot
    (t\eta)^{-r/(r+1)}
    + \sigma^2\cdot \frac{(t\eta)^{1/(r+1)}}{n}\notag\\
    &\qquad+(\omega^2 + \sigma^2)\cdot \eta\cdot\log(t)\cdot\bigg[\log(n)+\frac{\log^{5/2}(n)}{n^{1/2}}\cdot (t\eta)^{1/(r+1)}\bigg].
\end{align*}
\end{corollary}

Corollary \ref{coro:SGD_error_polydecay} provides concrete excess risk bounds for SGD and GD, based on which we can make a comparison between SGD and GD in terms of their iteration and gradient complexities.
For simplicity, in the following discussion, we assume that $\omega^2 \eqsim \sigma^2 \eqsim 1$.
Then choosing $t \eta \eqsim n$ minimizes the upper bound for GD risk and yields the $O(n^{-r / (r+1)})$ rate. Here GD can employ a constant stepsize.
Similarly, SGD can match the GD's rate, $O(n^{-r / (r+1)})$, by setting $ t\eta \eqsim n$ and 
\begin{equation}\label{eq:sgd_opt_lr}
    \eta \lesssim  \log^{-1} (t) \cdot \min\{\log^{-1}(n)\cdot n^{-\frac{r}{r+1}}, \ \log^{-\frac{5}{2}}(n) \cdot n^{ -\frac{1}{2}  })\}.
\end{equation}
The above stepsize choice implies that that SGD (fixed stepsize, last iterate) can only cooperate with small stepsize.

\paragraph{Iteration Complexity.}
We first compare GD and SGD in terms of the iteration complexity.
To reach the optimal rate, GD can employ a constant stepsize and set the number of iterates to be $t \eqsim n$.
However, in order to shelve the fluctuation error, the stepsize of SGD cannot be large, as required by \eqref{eq:sgd_opt_lr}.
More precisely, in order to match the optimal rate, SGD needs to use a small stepsize, $\eta \eqsim n /t$, with a large number of iterates,
\begin{equation*}
    t \eqsim \begin{cases}
    \log(n) \cdot  n^{1+\frac{r}{r+1}} = \tilde{\cO}( n^{1+\frac{r}{r+1}}), & r>1; \\
    \log^{3.5}(n) \cdot n^{1.5} = \tilde{\cO}(n^{1.5}), & r \le 1.
    \end{cases}
\end{equation*}
It can be seen that the iteration complexity of SGD is much worse than that of GD.
This result is empirically verified by Figure \ref{fig:iter-comparison} (a).

\paragraph{Gradient Complexity.}
We next compare  GD and SGD in terms of the gradient complexity.
Recall that for each iterate, GD computes $n$ gradients but SGD only computes $1$ gradient. 
Therefore, to reach the optimal rate, the total number of gradient computed by GD needs to be $\Theta( n^2) $, but that computed by SGD is only $\tilde{\cO} (n^{\max\{(2r+1)/ (r+1), 1.5 \}})$.
Thus, the gradient complexity of SGD is better than that of GD by a factor of $\tilde \cO(n^{\min\{0.5, 1/(r+1)\}})$.
This result is empirically verified by Figure \ref{fig:iter-comparison} (b).

\begin{figure}
    \centering
    \subfigure[\small{Iteration Complexity vs.  Risk}]{\includegraphics[width=0.45\textwidth]{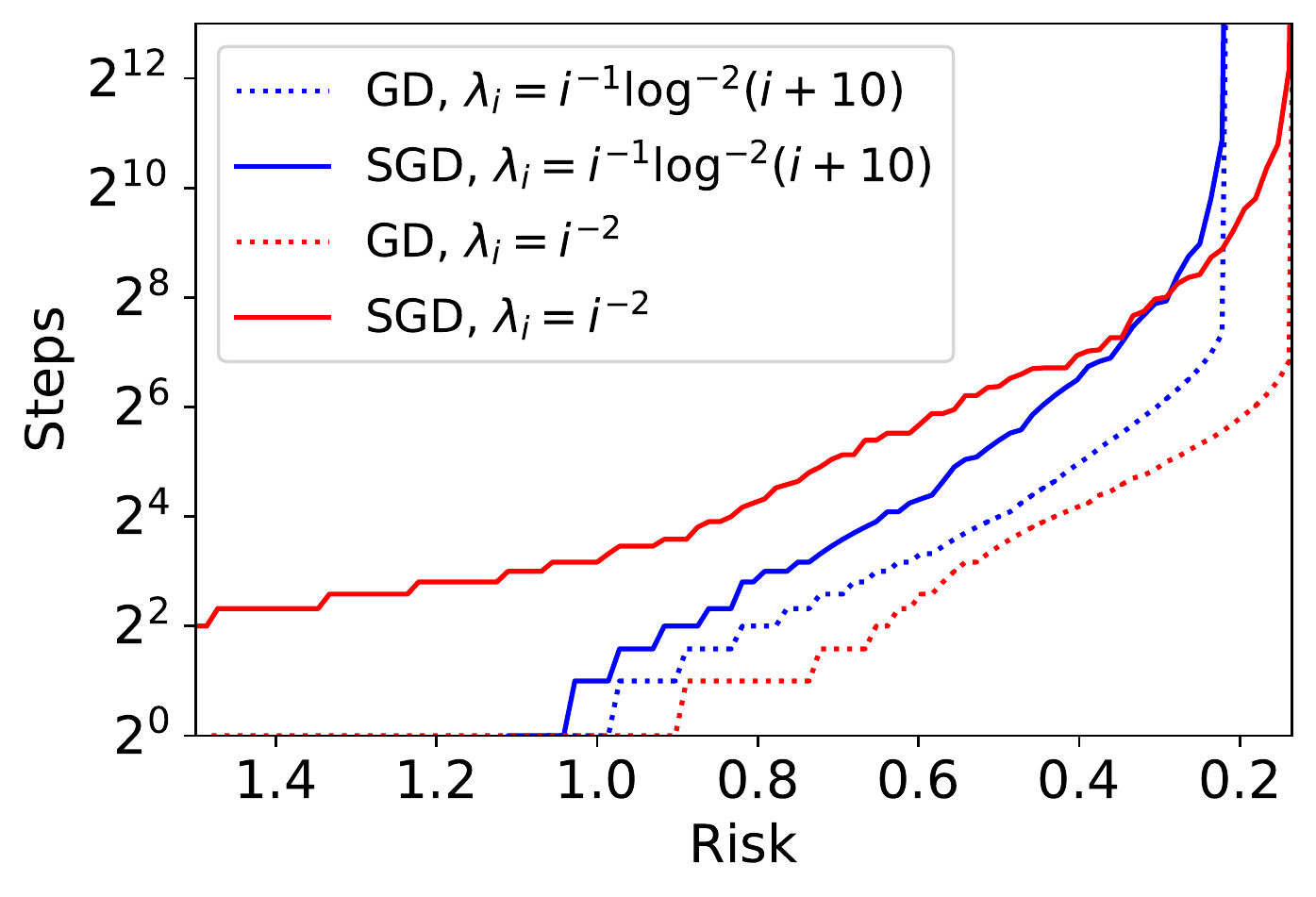}}
      \subfigure[\small{Gradient Complexity vs.  Risk}]{\includegraphics[width=0.45\textwidth]{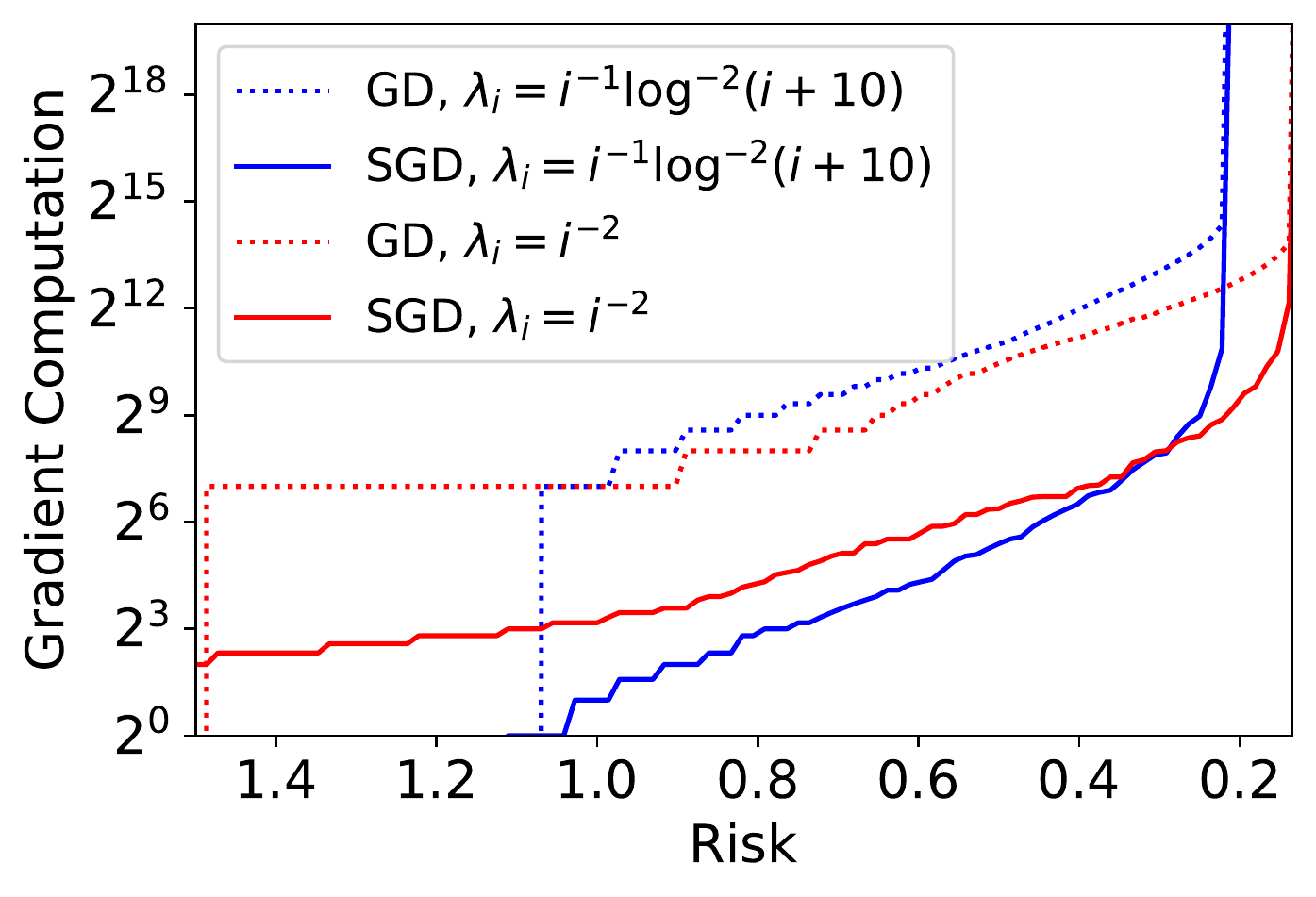}}
    \caption{\small 
      Iteration and gradient complexity comparison between SGD and GD.
      The curves report the minimum number of steps/gradients for each algorithm (with an optimally tuned stepsize) to achieve a targeted risk. Experiment setup is the same as that in Figure \ref{fig:risk-comparison}.
      }
    \label{fig:iter-comparison}
    \vspace{-.6cm}
\end{figure}

\section{Overview of the Proof Technique}

Our proof technique is inspired by the operator methods for analyzing single-pass SGD \citep{bach2013non,dieuleveut2017harder,jain2017markov,jain2017parallelizing,neu2018iterate,ge2019step,zou2021benign,wu2021last}.
In particular, they track an error \emph{matrix}, $(\wb_t-\wb^*)\otimes(\wb_t-\wb^*)$ that keeps richer information than the error norm $\|\wb_t-\wb^*\|_2^2$. 
For single-pass SGD where each data is used only once, the resulted iterates enjoy a simple dependence on history that allows an easy calculation of the expected error matrix (with respect to the randomness of data generation).
However for multi-pass SGD, a data might be used multiple times, which prevents us from tracking the expected error matrix directly.
Instead, a trackable analogy to the error matrix is the \emph{empirical error matrix}, $(\wb_t- \hat \wb)\otimes(\wb_t- \hat \wb)$ where $\hat{\wb}$ is the minimum norm interpolator.
More precisely, note that
\begin{align}\label{eq:update_sgd_error_vector}
\wb_{t+1} -\hat\wb&= \wb_t -\hat\wb -\eta\cdot (\xb_{i_t}\xb_{i_t}^\top\wb_t - \xb_{i_t}\xb_{i_t}^\top\hat \wb) = (\Ib - \eta \xb_{i_t}\xb_{i_t}^\top) (\wb_{t} -\hat\wb).
\end{align}
Therefore the expected (over the algorithm's randomness) empirical error matrix enjoy a simple update rule:
\begin{align*}
\text{let}\ \Eb_t := \EE_{\sgd}\big[(\wb_{t} -\hat\wb)(\wb_t -\hat\wb)^\top\big],\ \text{then}\ \Eb_{t+1} = \EE_{i_t} \big[(\Ib - \eta \xb_{i_t} \xb_{i_t}^\top ) \Eb_t (\Ib - \eta \xb_{i_t} \xb_{i_t}^\top )\big].
\end{align*}
Let $\bSigma := \frac{1}{n} \Xb^\top \Xb$ be the empirical covariance matrix.
We then follow the operator method \citep{zou2021benign} to define the following operators on symmetric matrices:
\[
\cG := (\Ib - \eta \bSigma) \otimes (\Ib - \eta \bSigma),\ \ 
\cM := \EE_{\sgd} [\xb_{i_t}\otimes \xb_{i_t}\otimes \xb_{i_t} \otimes \xb_{i_t}], \ \ 
\tilde{\cM} := \bSigma \otimes \bSigma.
\]
One can verify that, 
for a symmetric matrix $\Jb$, the following holds:
\begin{align*}
\cG\circ\Jb :&= (\Ib-\eta\bSigma)\Jb(\Ib-\eta\bSigma),\ \   \cM\circ\Jb:= \EE_{\sgd} [\xb_{i_t}\xb_{i_t}^\top\Jb\xb_{i_t}\xb_{i_t}^\top],\ \ \tilde \cM\circ\Jb:=\bSigma\Jb\bSigma.
\end{align*}
Moreover, the following properties of the defined operators are essential in the subsequent analysis:
\begin{itemize}[leftmargin=*]
    \item  \textbf{PSD mapping:} for every PSD matrix $\Jb$, $\cM\circ\Jb$, $(\cM-\tilde \cM)\circ\Jb$ and $\cG\circ\Jb$ are all PSD matrices.
    \item \textbf{Commutative property:} for two  PSD matrices $\Bb_1$ and $\Bb_2$, we have 
    \begin{align*}
\la\cG \circ\Bb_1, \Bb_2\ra = \la\Bb_1, \cG \circ\Bb_2\ra,\ \la\cM \circ\Bb_1, \Bb_2\ra = \la\Bb_1, \cM \circ\Bb_2\ra, \ \la\tilde\cM \circ\Bb_1, \Bb_2\ra = \la\Bb_1, \tilde\cM \circ\Bb_2\ra
\end{align*}
\end{itemize}
Based on these operators, we can obtain a close form update rule for $\Eb_t$:
\begin{align}
\Eb_{t} &= \EE_{i_{t-1}} (\Ib - \eta \xb_{i_{t-1}} \xb_{i_{t-1}}^\top ) \Eb_t (\Ib - \eta \xb_{i_{t-1}} \xb_{i_{t-1}}^\top )\notag\\
& = \cG\circ\Eb_{t-1} + \eta^2\cdot(\cM-\tilde\cM)\circ\Eb_{t-1}\notag\\
&=\underbrace{\cG^{t}\circ\Eb_0}_{\bTheta_1} + \underbrace{\eta^2\cdot \sum_{k=0}^{t-1}\cG^{t-1-k}\circ(\cM-\tilde\cM)\circ \Eb_{k}}_{\bTheta_2}.\label{eq:update_form_Et}
\end{align}
Here the first term 
\[\bTheta_1 := (\Ib-\eta\bSigma)^t\Eb_0(\Ib-\eta\bSigma)^t = (\hat{\wb}_t - \hat{\wb})(\hat{\wb}_t - \hat{\wb})^\top\] 
is exactly the error matrix caused by GD iterates (with stepsize $\eta$ and iteration number $t$), and the second term $\bTheta_2$ is a \textit{fluctuation matrix} that captures the deviation of a SGD iterate $\wb_t$ with respect to a corresponding GD iterate $\hat{\wb}_t$.
We remark that the expected error matrix $\Eb_t$ contains all information of $\wb_t$.

We next prove Theorem \ref{thm:risk_decomposition}, from where we will see the usage of $\Eb_t$.


\subsection{Risk Decomposition: Proof of Theorem \ref{thm:risk_decomposition}}
The following fact is clear from the update rule \eqref{eq:update_sgd_error_vector}.
\begin{fact}\label{fact:sgd_expectation}
The GD iterates satisfy
\( \hat{\wb}_{t+1} -\hat\wb = (\Ib - \eta \bSigma) (\hat{\wb}_{t} -\hat\wb) \) and  \( \EE_{\sgd}[\wb_t-\hat\wb] = \hat\wb_t - \hat\wb \).
\end{fact}
Based on Fact \ref{fact:sgd_expectation} and \eqref{eq:update_form_Et}, we have
\begin{align*}
&\EE_{\sgd}[(\wb_t - \wb^*)(\wb_t - \wb^*)^\top] \notag\\
&=  \Eb_t + (\hat\wb - \wb^*)(\hat\wb_t - \hat\wb)^\top + (\hat\wb_t - \hat\wb)(\hat\wb - \wb^*)^\top + (\hat\wb-\wb^*)(\hat\wb-\wb^*)^\top\notag\\
&= \bTheta_1 +  (\hat\wb - \wb^*)(\hat\wb_t - \hat\wb)^\top + (\hat\wb_t - \hat\wb)(\hat\wb - \wb^*)^\top + (\hat\wb-\wb^*)(\hat\wb-\wb^*)^\top +\bTheta_2\notag\\
& = (\hat\wb_t - \wb^*)(\hat\wb_t - \wb^*)^\top + \bTheta_2,
\end{align*}
where $\bTheta_1$ and $\bTheta_2$ are defined in \eqref{eq:update_form_Et} and the last equality is due to $\bTheta_1 = (\hat\wb_t-\hat\wb)(\wb_t-\hat\wb)^\top$.
Also note that
\begin{align*}
\EE_{\sgd}[\cE(\wb_t)] = \frac{1}{2}\EE_{\sgd}\big[\|\wb_t - \wb^*\|_\Hb^2\big] = \frac{1}{2}\big\la\EE_{\sgd}[(\wb_t - \wb^*)(\wb_t - \wb^*)^\top], \Hb\big\ra.
\end{align*}
Combining these two inequalities proves Theorem \ref{thm:risk_decomposition}:
\begin{align} \label{eq:risk_decomposition_detailed}
\EE_{\sgd} [ \cE(\wb_t)] = \underbrace{\frac{1}{2}\|\hat\wb_t - \wb^*\|_{\Hb}^2}_{\text{GD error}} +    \underbrace{\frac{\eta^2}{2}\cdot \sum_{k=0}^{t-1}\big\la\cG^{t-1-k}\circ(\cM-\tilde\cM)\circ \Eb_{k}, \Hb\big\ra}_{\text{Fluctuation error }}.
\end{align}
Finally, the fluctuation error is also positive because both $\cG$ and $\cM - \tilde\cM$ are PSD mappings.

\subsection{Bounding the Fluctuation Error: Proof of Theorem \ref{thm:upperbound_fluctuation}}\label{sec:bounding_fluc_err}
There are several challenges in the analysis of fluctuation error: (1) it is difficult to characterize the matrix $(\cM  - \tilde \cM) \circ\Eb_k$ since the matrix $\Eb_k$ is unknown; (2) the operator $\cG$ involves an exponential decaying term with respect to the empirical covariance matrix $\bSigma$, which does not commute with the population covariance matrix $\Hb$. 

To address the first problem, we will use the PSD mapping and 
commutative property of the operators $\tilde \cM$, $\cG$, $\cM$ and obtain the following result.
\begin{align}\label{eq:formula_fluctuationErr}
\mathrm{FluctuationError} \le \frac{\eta^2}{2}\cdot \sum_{k=0}^{t-1}\la\cM\circ\cG^{t-1-k}\circ\Hb, \Eb_k\ra.
\end{align}
Now, the input of the operator $\cM\circ \cG^{t-1-k}$ will not be an unknown matrix but a fixed one (i.e., $\Hb$), and the remaining effort will be focusing on characterizing  $\cM\circ \cG^{k}\circ\Hb$.
Applying the definitions of $\cM$ and $\cG$ implies
\begin{align*}
\cM\circ\cG^k\circ\Hb = \EE_{i}\big[\xb_i\xb_i^\top (\Ib-\eta\bSigma)^k\Hb(\Ib-\eta\bSigma)^k\xb_i\xb_i^\top\big].
\end{align*}
Then our idea is to first prove an uniform upper bound on the quantity $\xb_i^\top (\Ib-\eta\bSigma)^k\Hb(\Ib-\eta\bSigma)^k\xb_i$ for all $i\in[n]$ (e.g., denoted as $U(k, \eta, n)$), then it can be naturally obtained that
\begin{align}\label{eq:upperbound_fourthmoment}
\cM\circ\cG^k\circ\Hb\preceq U(k,  \eta, n)\cdot \EE_{i}[\xb_i\xb_i^\top] = U(k,\eta, n)\cdot\bSigma,
\end{align}
then we will only need to characterize the inner product $\la\Eb_k,\bSigma\ra$ in \eqref{eq:formula_fluctuationErr}, which can be understood as the optimization error at the $k$-th iteration.

In order to precisely characterize $U(k, \eta, n)$, we encounter the second problem that the population covariance $\Hb$ and empirical covariance $\bSigma$ are not commute, thus the exponential decaying term $(\Ib-\eta\bSigma)^k$ will not be able to fully decrease $\Hb$ since some components of $\Hb$ may lie in the small eigenvalue directions of $\bSigma$. Therefore, we consider the following decomposition
\begin{align*}
\xb_i^\top (\Ib-\eta\bSigma)^k\Hb(\Ib-\eta\bSigma)^k\xb_i = \underbrace{\xb_i^\top (\Ib-\eta\bSigma)^k\bSigma(\Ib-\eta\bSigma)^k\xb_i}_{\Theta_1} + \underbrace{\xb_i^\top (\Ib-\eta\bSigma)^k(\Hb-\bSigma)(\Ib-\eta\bSigma)^k\xb_i}_{\Theta_2}.
\end{align*}
Then for $\Theta_1$, it can be seen that the decaying term $(\Ib-\eta\bSigma)^k$ is commute with $\bSigma$ thus can successfully make it decrease. For $\Theta_2$, we  will view the difference $\Hb-\bSigma$ as the component of $\Hb$ that cannot be effectively decreased by $(\Ib-\eta\bSigma)^k$, which will be small  as $n$ increases.

More specifically, we can get the following upper bound on $\Theta_1$.
\begin{lemma}\label{lemma:upperbound_Theta1}
If the stepsize satisfies $\gamma\le c/\tr(\Hb)$ for some small absolute constant $c$, then with probability at least $1-1/\poly(n)$, it holds that
\begin{align*}
\Theta_1 \lesssim \tr(\Hb) \cdot \log(n)\cdot \min\bigg\{\frac{1}{(k+1)\eta}, \|\Hb\|_2\bigg\}. 
\end{align*}
\end{lemma}

For $\Theta_2$, we will rewrite $\xb_i$ as $\eb_i^\top\Xb$ where $\eb_i\in\RR^n$ and $\Xb\in\RR^{n\times d}$, then
\begin{align}\label{eq:upperbound_Theta2_main}
\Theta_2 &= \eb_i^\top \Xb(\Ib-\eta\bSigma)^k(\Hb-\bSigma)(\Ib-\eta\bSigma)^k\Xb^\top\eb_i\notag\\
&\le \|\eb_i^\top\Xb(\Ib-\eta\bSigma)^k\|_2^2\cdot\|\Hb-\bSigma\|_2.
\end{align}
Then since $\Xb$ and $\bSigma$ have the same column eigenspectrum, we can fully unleash the decaying power of the term $(\Ib-\eta\bSigma)^k$ on $\Xb$. Further note the that the row space of $\Xb$ is uniform distributed (corresponding to the index of training data), which is independent of $\eb_i$. This implies that we can adopt standard concentration arguments with covering on $n$ fixed vectors $\{\eb_i\}_{i=1,\dots,n}$ to prove a sharp high probability upper bound (compared to the naive worst-case upper bound). Consequently, we state the upper bound on $\Theta_2$ in the following lemma.
\begin{lemma}\label{lemma:upperbound_Theta2}
For every $i\in[n]$, we have with probability at least $1-1/\poly(n)$, the following holds for every $k^*\in[d]$,
\begin{align}
\Theta_2\lesssim \frac{ \log^{5/2}(n)}{n^{1/2}}\cdot  \bigg(\frac{k^*}{(k+1)\eta} + \sum_{i>k^*}\lambda_i\bigg).
\end{align}
\end{lemma}

\subsection{Bounding the Risk of GD: Proof of Theorem \ref{thm:GD_earlystop}}\label{sec:gd_risk}
Recall that 
\(
\hat{\wb} = \Xb^\top (\Xb \Xb^\top)^{-1}\yb =  \Xb^\top \Ab^{-1}\yb \),
where $\Ab := \Xb \Xb^\top$ is the gram matrix.
Then we can reformulate $\hat\wb_t$ by
\begin{align*}
\hat{\wb}_t &= \hat{\wb} - (\Ib - \eta \bSigma)^t (\hat{\wb}_0 - \hat{\wb})=  \big( \Ib - (\Ib - \eta \bSigma)^t \big)\Xb^\top \Ab^{-1}\yb= \Xb^\top \big( \Ib - (\Ib - \eta n^{-1}\Ab)^t \big) \Ab^{-1}\yb.
\end{align*}
Denote $\tilde{\Ab} := \Ab \big( \Ib - (\Ib - \eta n^{-1}\Ab)^t \big)^{-1}$, the excess risk of $\hat\wb_t$ is 
\begin{align}\label{eq:bias_var_decomposition_GD}
    \cE(\hat\wb_t)
    &= \frac{1}{2}\big\| \Xb^\top \tilde{\Ab}^{-1} \yb - \wb^* \big\|^2_{\Hb}= \underbrace{\frac{1}{2}\big\|\wb^*\big(\Ib-\Xb^\top\tilde\Ab^{-1}\Xb\big)\big\|_\Hb^2}_{\mathrm{BiasError}} + \underbrace{\frac{1}{2}\big\|\Xb^\top\tilde\Ab^{-1}\bepsilon\big\|_\Hb^2}_{\mathrm{VarError}} . 
\end{align}
The remaining proof will be relates the excess risk of early stopped GD to that of ridge regression with certain regularization parameters. In particular, note that the excess risk of the ridge regression solution with parameter $\lambda$ is $\frac{1}{2}\|\Xb^\top(\Ab+\lambda\Ib)^{-1}\yb-\wb^*\|_\Hb^2$. 
Then it remains to show the relationship between $\tilde\Ab$ and $\Ab+\lambda\Ib$, which is illustrated in the following lemma.
\begin{lemma}\label{lemma:tildeA-bounds}
For any $\eta\le c/\lambda_1$ for some absolute constant $c$ and $t>0$, we have
\[
\frac{1}{2} \bigg( \Ab + \frac{n}{\eta t} \Ib \bigg) \preceq \tilde{\Ab} \preceq\Ab + \frac{2 n}{t\eta}\cdot \Ib.
\]
\end{lemma}
Then, the lower bound of $\tilde \Ab$ will be applied to prove the upper bound of variance error of GD, as shown in \eqref{eq:bias_var_decomposition_GD}, which is at most four times the variance error achieved by the ridge regression with $\lambda = n/(\eta t)$. The upper bound of $\tilde \Ab$ will be applied to prove the upper bound of the bias error of GD, which is at most the bias error achieved by ridge regression with $\lambda = 2n/(\eta t)$.  Finally, we can apply the prior work \citep[Theorem 1]{tsigler2020benign} on the excess risk analysis for ridge regression to complete the proof for bounding the bias and variance errors separately.

\section{Conclusion and Discussion}
In this paper, we establish an instance-dependent excess risk bound of multi-pass SGD for interpolating least square problems. The key takeaways include: (1) the excess risk of SGD is \textit{always} worse than that of GD, given the same setup of stepsize and iteration number; (2) in order to achieve the same level of excess risk, SGD requires more iterations than GD; and (3) however, the gradient complexity of SGD can be better than that of GD.
The proposed technique for analyzing multi-pass SGD could be of broader interest.

Several interesting problems are left for future exploration:
\paragraph{A problem-dependent excess risk lower bound} could be useful to help understand the sharpness of our excess risk upper bound for multi-pass SGD.
The challenge here is mainly from
the fact that the empirical covariance matrix $\bSigma$ does not commute with the population covariance matrix $\Hb$. 
In particular, one needs to develop an even sharper characterization on the quantity $\cM\circ\cG^k\circ\Hb$ (see Section \ref{sec:bounding_fluc_err}); more precisely, a sharp lower bound on $\xb_i^\top (\Ib-\eta\bSigma)^k\Hb(\Ib-\eta\bSigma)^k\xb_i$ is required.

\noindent \textbf{Multi-pass SGD without replacement} is a more practical SGD variant than the multi-pass SGD with replacement studied in this work. 
The key difference is that, the former does not pass training data independently (since each data must be used for equal times).
In terms of optimization complexity, it has already been demonstrated in theory that multi-pass SGD without replacement (e.g., SGD with single shuffle or random shuffle) outperforms multi-pass SGD with replacement \citep{haochen2019random,safran2020good,ahn2020sgd}. 
In terms of generalization, it is still open whether or not the former can be better than the latter, as there lacks a sharp excess risk analysis for multi-pass SGD without replacement. 
The techniques presented in this paper can shed light on this direction.


\appendix




\section{Risk Bound for the Fluctuation Error}


\subsection{Proof of \eqref{eq:formula_fluctuationErr}}
\begin{lemma}\label{lemma:formula_fluctuationErr}
The fluctuation error satisfies
\begin{align*}
\mathrm{FluctuationError} \le \frac{\eta^2}{2}\cdot \sum_{k=0}^{t-1}\la\cM\circ\cG^{t-1-k}\circ\Hb, \Eb_k\ra.
\end{align*}
\end{lemma}
\begin{proof}[Proof of Lemma \ref{lemma:formula_fluctuationErr}]
By Lemma \ref{eq:risk_decomposition_detailed}, we have
\begin{align*}
\mathrm{FluctuationError}= \frac{\eta^2}{2}\cdot\sum_{k=0}^{t-1}\la\cG^{t-1-k}\circ(\cM-\tilde \cM)\circ\Eb_k, \Hb\ra.
\end{align*}
Then note that $\cM$, $\cM-\tilde \cM$ and $\cG$ are the PSD mapping. Then we have
\begin{align*}
\cG^{t-1-k}\circ(\cM-\tilde \cM)\circ\Eb_k\preceq \cG^{t-1-k}\circ\cM\circ\Eb_k
\end{align*}
for all $k\ge 0$. Further using the commutative property of $\cG$ and $\cM$, we have
\begin{align*}
\la\cG^{t-1-k}\circ\cM\circ\Eb_k, \Hb\ra = \la\cM\circ\cG^{t-1-k}\circ\Hb,\Eb_k\ra.
\end{align*}
This completes the proof.

\end{proof}

\subsection{Proof of Lemma \ref{lemma:upperbound_Theta1}}

We first present the following two useful lemmas.
\begin{lemma}[Theorem 9 in \citet{bartlett2020benign}]\label{lemma:concentration_covariance} There is an absolute constant $c$ such that for any $\delta\in(0,1)$ with probability at least $1-\delta$,
\begin{align*}
\|\bSigma - \Hb\|_2\le c\|\Hb\|_2\cdot\max\bigg\{\sqrt{\frac{r(\Hb)}{n}}, \frac{r(\Hb)}{n}, \sqrt{\frac{\log(1/\delta)}{n}}, \frac{\log(1/\delta)}{n}\bigg\},
\end{align*}
where $r(\Hb) = \sum_{i}\lambda_i/\lambda_1$.

\end{lemma}

\begin{lemma}[Lemma 22 in \cite{bartlett2020benign}]\label{lemma:high_prob_sum_subexp} There is a universal constant $c$ such that for any independent, mean zero, $\sigma$-subexponential random variables $\xi_1,\dots,\xi_n$, any $\ab = (a_1,\dots,a_n)$ and any $t\ge 0$,
\begin{align*}
\PP\bigg(\bigg|\sum_{i=1}^n a_i\xi_i\bigg|\ge t\bigg)\le 2\exp\bigg[-c\min\bigg(\frac{t^2}{\sigma^2\|\ab\|_2^2}, \frac{t}{\sigma\|\ab\|_\infty}\bigg)\bigg].
\end{align*}

\end{lemma}

\begin{proof}[Proof of Lemma \ref{lemma:upperbound_Theta1}]
 Note that $(1-x)^k\le 1/[x(k+1)]$ for all $k>0$ and $x\in(0, 1)$, we have
\begin{align*}
(\Ib-\eta\bSigma)^k\bSigma(\Ib-\eta\bSigma)^k = \bSigma(\Ib-\eta\bSigma)^{2k} \preceq  \frac{1}{2(k+1)\eta}\cdot\Ib.
\end{align*}
Besides, we also have $\bSigma(\Ib-\eta\bSigma)^{2k}\preceq \bSigma$. This implies that
\begin{align}\label{eq:bound_theta1_temp}
\Theta_1 \le \min\bigg\{\xb_i^\top\bSigma\xb_i,\frac{\|\xb_i\|_2^2}{2(k+1)\eta} \bigg\}\le \min\bigg\{ \|\bSigma\|_2\cdot\|\xb_i\|_2^2,\frac{\|\xb_i\|_2^2}{2(k+1)\eta} \bigg\}.
\end{align}
Then applying Lemma \ref{lemma:concentration_covariance} and using the assumption that $\lambda_1=\Theta(1)$, we have
\begin{align*}
\|\bSigma\|_2\lesssim \|\Hb\|_2.
\end{align*}
Besides, by Assumption \ref{assump:data_distribution}, we have
\begin{align*}
\|\xb_i\|_2^2 = \sum_{i}\lambda_i\cdot z_i^2
\end{align*}
where $z_i$ is independent $1$-subgaussian random variable  and satisfies $\EE[z_i^2]=1$. Therefore, applying Lemma \ref{lemma:high_prob_sum_subexp} we can get with probability $1-\delta$,
\begin{align*}
\|\xb_i\|_2^2\lesssim \sum_{i}\lambda_i + \max\bigg\{\log(1/\delta)\cdot\lambda_1, \sqrt{\log(1/\delta)\sum_i\lambda_i^2}\bigg\}.
\end{align*}
Setting $\delta = 1/\poly(n)$ and applying union bound over all $i\in[n]$, we can get with probability at least $1-1/\poly(n)$, it holds that $\|\xb_i\|_2^2\le \log(n)\cdot \tr(\Hb)$ for all $i\in[n]$. Putting  this into \eqref{eq:bound_theta1_temp} completes the proof.

\end{proof}

\subsection{Proof of Lemma \ref{lemma:upperbound_Theta2}}
We first provide the following useful facts and lemmas.
\begin{fact}[Part of Lemma 8 in \citet{bartlett2020benign}]\label{fact:decomposition_gram}
The gram matrix $\Ab=\Xb\Xb^\top$ can be decomposed by
\begin{align*}
\Ab = \sum_{i}\lambda_i\zb_i\zb_i^\top,
\end{align*}
where $\zb_i\in\RR^n$ are independent $1$-subgaussian random vector satisfying $\EE[\|\zb_i\|_2^2]=n$.
\end{fact}

\begin{fact}\label{fact:covariance2gram}
Assume $n<d$ and the gram matrix $\Ab$ is of full-rank, then it holds that
\begin{align*}
\Xb(\Ib_d-\eta\bSigma)^k = (\Ib_n-\eta n^{-1}\Ab)^k\Xb.
\end{align*}

\end{fact}
\begin{proof}[Proof of Fact \ref{fact:covariance2gram}]
Note that $\Xb\in\RR^{n\times d}$, consider its SVD decomposition $\Xb = \Ub\bLambda\Vb^\top$, where $\Ub\in\RR^{n\times n}$, $\Vb\in\RR^{d\times d}$ and $\bLambda\in\RR^{n\times d}$. Then we have $\bSigma = n^{-1}\Xb^\top\Xb = n^{-1}\Vb\bLambda^\top\bLambda\Vb^\top$, which implies that
\begin{align*}
\Xb(\Ib-\eta\bSigma)^k = \Ub\bLambda\Vb^\top\Vb(\Ib_d-\eta n^{-1} \bLambda^\top\bLambda)^k\Vb^\top = \Ub\bLambda(\Ib_d-\eta n^{-1} \bLambda^\top\bLambda)^k\Vb^\top.
\end{align*}
Additionally, it is easy to verify that $\bLambda(\Ib_d-\eta n^{-1}\bLambda^\top\bLambda) = (\Ib_n-\eta n^{-1} \bLambda\bLambda^\top)\bLambda$. Therefore, it follows that
\begin{align*}
\Xb(\Ib-\eta\bSigma)^k = \Ub\bLambda(\Ib_d-\eta n^{-1}\bLambda^\top\bLambda)^k\Vb^\top = \Ub(\Ib_n-\eta n^{-1}\bLambda\bLambda^\top)^k \bLambda\Vb^\top = (\Ib_n - \eta\Ab)^k\Xb,
\end{align*}
where the last equality follows from the fact that $\Ab = \Ub\bLambda\bLambda^\top\Ub^\top$. This completes the proof.

\end{proof}

\begin{lemma}\label{lemma:highprob_matrixnorm_randunit}
Let $\ub\in\cS^{n-1}$ be a uniformly random unit vector,  then for any fixed PSD matrix $\bTheta\in\RR^{n\times n}$, with probability at least $1-1/\poly(n)$, it holds that
\begin{align*}
\ub^\top\bTheta\ub  \lesssim \frac{\log(n)}{n}\cdot\tr(\bTheta).
\end{align*}
\end{lemma}
\begin{proof}
We first consider a Gaussian random vector $\vb\sim N(0, \Ib_n/n)$, then it is clear that we can reformulate it as $\vb = r\cdot \ub$, where $\ub$ is a uniformly random unit vector and $\EE[r] = 1$. Note that $nr$ follows $\chi^2(n)$ distribution, then with probability at least $1-e^{-c n}$ for some small constant $c>0$ we have $r\ge 1/2$. Moreover, let $\bTheta=\sum_i\mu_i\zb_i\zb_i^\top$ be the eigen-decomposition of $\bTheta$, we have
\begin{align*}
n\vb^\top\bTheta\vb-\tr(\bTheta) = \sum_{i=1}^n \mu_i[n(\zb_i^\top\vb)^2 - 1]: = \sum_{i=1}^n \mu_i\xi_i
\end{align*}
where $\xi_i\sim \chi^2(1)-1$ distribution, which is $1$-subexponential. Then applying Lemma \ref{lemma:high_prob_sum_subexp}, we have with probability at least $1-2e^{-x}$ such that
\begin{align*}
\sum_{i=1}^n \mu_i\xi_i\le C\cdot \max\bigg(x\mu_1, \sqrt{x\sum_{i=1}^n\mu_i^2} \bigg)
\end{align*}
holds for some constant $C$.

Combining the previous results, we have with probability at least $1-e^{cn}-2e^{-x}$,
\begin{align*}
\ub^\top\bTheta\ub = r^{-1}\vb^\top\bTheta\vb\le \frac{2}{n}\bigg[\tr(\bTheta) + C\cdot \max\bigg(x\mu_1, \sqrt{x\sum_{i=1}^n\mu_i^2} \bigg)\bigg].
\end{align*}
Further note that $\sum_{i=1}^n\mu_i^2, \mu_1\le \tr^2(\bTheta)$, then setting $x = C'\log(n)$ for some absolute constant $C'$, we have with probability at least $1-1/\mathrm{poly}(n)$, 
\begin{align*}
\ub^\top\bTheta\ub = r^{-1}\vb^\top\bTheta\vb\le \frac{C''\log(n)}{n}\cdot\tr(\bTheta)
\end{align*}
for some absolute constant $C''$. This completes the proof.

\end{proof}

\begin{lemma}\label{lemma:trace_Aexp}
For any $k^*\in[d]$, with probability at least $1-1/\poly(n)$, it holds that
\begin{align*}
\tr(\Ab(\Ib_n-\eta n^{-1}\Ab)^{2k}) \lesssim \frac{n k^*}{(k+1)\eta} + n\log(n)\cdot\sum_{i>k^*}\lambda_i.
\end{align*}
\end{lemma}
\begin{proof}
Let $\mu_1,\dots,\mu_n$ be the sorted (in descending order) eigenvalues of $\Ab$, then we have
\begin{align}\label{eq:decomposition_trace}
\tr\big(\Ab(\Ib_n-\eta n^{-1}\Ab)^{2k})\big)= \sum_{i=1}^n \mu_i\cdot (1-\eta n^{-1}\mu_i)^{2k} \le \sum_{i=1}^n\min\bigg\{\frac{n}{2(k+1)\eta}, \mu_i\bigg\},
\end{align}
where the inequality follows from the fact that $(1-x)^k\le 1/[(k+1)x]$ for all $x\in(0, 1)$ and $k>0$.
Additionally, by Fact \ref{fact:decomposition_gram} we have
\begin{align*}
\Ab = \sum_{i}\lambda_i\zb_i\zb_i^\top,
\end{align*}
where $\{\zb_i\}_{i=1,\dots,n}$ are i.i.d. $1$-subgaussian random vectors satisfying $\EE[\zb_i]=0$ and $\EE[\|\zb_i\|_2^2]=n$. Then define
\begin{align}\label{eq:def_Ak}
\Ab_k :=\sum_{i>k}\lambda_i\zb_i\zb_i^\top,
\end{align}
and 
\begin{align*}
\Ab_k = \sum_{i=1}^n \mu_i(\Ab_k)\ub_i\ub_i^\top
\end{align*}
be its eigen-decomposition. Then note that $\Ab-\Ab_k+\sum_{i=1}^j\mu_i(\Ab_k)\ub_i\ub_i^\top$ has rank at most $k+j$, thus there must exist a linear space $\cL$ of dimension $n-k-j$ (that is orthogonal to $\{\zb_i\}_{i=1,\dots,k}$ and $\{\ub_i\}_{i=1}^j$) such that for all $\vb\in\cL$, 
\begin{align*}
\vb^\top\Ab\vb\le\vb^\top\mu_1\bigg(\Ab_k-\sum_{i=1}^j\mu_i(\Ab_k)\ub_i\ub_i^\top\bigg)\vb = \vb^\top\mu_{j+1}(\Ab_k)\vb. 
\end{align*}
This implies that for any $k\in[n]$ and $j\in[n-k]$, it holds that
\begin{align*}
\mu_{k+j}(\Ab)\le \mu_j(\Ab_k),
\end{align*}
and thus
\begin{align}\label{eq:upperbound_tailsum_A}
\sum_{i=k+1}^n\mu_i\le \sum_{i=1}^{n+1-i}\mu_i(\Ab_k)\le \tr(\Ab_k).
\end{align}
Moreover, by the definition of $\Ab_k$ in \eqref{eq:def_Ak}, we have
\begin{align*}
\tr(\Ab_k) = \sum_{i>k}\lambda_i\|\zb_i\|_2^2.
\end{align*}
Then note that $\|\zb_i\|_2^2/n-1$ is $1$-subexponential, by Lemma \ref{lemma:high_prob_sum_subexp}, we have with probability at least $1-2e^{-x}$
\begin{align*}
\tr(\Ab_k) \le  n\sum_{i>k}\lambda_i + C\cdot n\cdot \max\bigg(x\lambda_{k+1},\sqrt{x\sum_{i>k}\lambda_i^2}\bigg).
\end{align*}
for some absolute constant $C$.
Then setting $x = \Theta(\log(n))$ and using the fact that $\sum_{i>k}\lambda_i^2\le (\sum_{i>k}\lambda_i)^2$, we have with probability at least $1-1/\poly(n)$,
\begin{align}\label{eq:upperbound_tr_Ak}
\tr(\Ab_k) \lesssim n\log(n)\cdot\sum_{i>k}\lambda_i.
\end{align}
Putting \eqref{eq:upperbound_tr_Ak} into \eqref{eq:upperbound_tailsum_A} and further applying \eqref{eq:decomposition_trace}, we have for any $k^*\in[n]$, with probability at least $1-1/\poly(n)$
\begin{align*}
\tr\big(\Ab(\Ib_n-\eta n^{-1}\Ab)^{2k})\big) \le  \sum_{i=1}^{k^*}\frac{n}{2(k+1)\eta}+\tr(\Ab_k) \lesssim \frac{n k^*}{(k+1)\eta} + n\log(n)\cdot\sum_{i>k^*}\lambda_i.
\end{align*}
This completes the proof.
\end{proof}


\begin{proof}[Proof of Lemma \ref{lemma:upperbound_Theta2}]
Recalling the formula of $\Theta_2$, we have
\begin{align*}
\Theta_2 = \xb_i^\top(\Ib-\eta\bSigma)^k(\Hb-\bSigma)(\Ib-\eta\bSigma)^k\xb_i.
\end{align*}
Moreover, note that $\xb_i$ can be rewritten as $\xb_i = \eb_i^\top\Xb$, where $\eb_i\in\RR^n$ and $\Xb\in\RR^{n\times d}$. Then
\begin{align}\label{eq:upperbound_Theta2_temp}
\Theta_2 &= \eb_i^\top \Xb(\Ib-\eta\bSigma)^k(\Hb-\bSigma)(\Ib-\eta\bSigma)^k\Xb^\top\eb_i \notag\\
&\le \|\eb_i^\top \Xb(\Ib-\eta\bSigma)^k\|_2^2\cdot \|\Hb-\bSigma\|_2.
\end{align}
Then by Fact \ref{fact:covariance2gram}, we have
\begin{align*}
\|\eb_i^\top \Xb(\Ib-\eta\bSigma)^k\|_2^2&=\|\eb_i^\top (\Ib_n-\eta n^{-1}\Ab)^k\Xb\| \notag\\
&= \eb_i^\top(\Ib_n-\eta n^{-1}\Ab)^k\Xb\Xb^\top(\Ib_n-\eta n^{-1}\Ab)^k\eb_i \notag\\
&= \eb_i^\top\Ab(\Ib_n-\eta n^{-1}\Ab)^{2k}\eb_i.
\end{align*}
Note that $\eb_i$ is independent of the randomness of $\Ab$ and the eigenvectors of $\Ab$ is rotation invariant. Specifically, note that $\Ab = \Ub\bLambda\bLambda^\top\Ub^\top$, where $\Ub\in\RR^{n\times n}$ is an orthonormal matrix and $\bLambda\bLambda^\top\in\RR^{n\times n}$ is an diagonal matrix. Then we consider the conditional distribution $\PP(\Ab|\bLambda\bLambda^\top)$, which can be viewed as a distribution over the orthonormal matrix $\Ub$, denoted by $\PP(\Ub)$. Then note that $\Ub$ can also be understood as a rotation matrix when operated on an vector, and using Fact \ref{fact:decomposition_gram}, we have for any rotation matrix $\Pb$, it holds that
\begin{align*}
\Pb\Ab\Pb^\top = \sum_{i}\lambda_i \Pb\zb_i\zb_i^\top\Pb^\top
\end{align*}
which has the same distribution of $\Ab = \sum_{i}\lambda_i \zb_i\zb_i^\top$ since $\Pb\zb_i$ and $\zb_i$ have the same distribution. Therefore, it can be verified that for any different orthonormal matrices $\Ub_1$ and $\Ub_2$ and let $\Pb = \Ub_2\Ub_1^\top$, which is also an orthonormal matrix, we have
\begin{align*}
\PP(\Ub_1\bLambda\bLambda^\top\Ub_1^\top|\bLambda\bLambda^\top) = \PP(\Pb\Ub_1\bLambda\bLambda^\top\Ub_1^\top\Pb^\top|\bLambda\bLambda^\top) = \PP(\Ub_2\bLambda\bLambda^\top\Ub_2^\top|\bLambda\bLambda^\top).
\end{align*}
This implies that $\PP(\Ub_1)=\PP(\Ub_2)$ for any $\Ub_1\neq \Ub_2$. Therefore, we can conclude that $\PP(\Ub)$ is an uniform distribution over the entire class of orthonormal matrices. Then note that 
\begin{align*}
\Ab(\Ib_n-\eta n^{-1}\Ab)^{2k} = \Pb\big(\bLambda\bLambda^\top(\Ib - n^{-1}\eta \bLambda\bLambda^\top)^{2k}\big)\Pb^\top.
\end{align*}
Then for any fixed $i$, using the fact that $\Pb$ is a uniformly random rotation matrix, we have $\Pb^\top\eb_i$ is a random unit vector in  $\cS^{n-1}$. Then applying Lemmas \ref{lemma:highprob_matrixnorm_randunit} and \ref{lemma:trace_Aexp}, and taking union bound over $i\in[n]$, we have with probability at least $1-1/\poly(n)$,
\begin{align}\label{eq:ei_expA}
\eb_i^\top\Ab(\Ib_n-\eta n^{-1}\Ab)^{2k}\eb_i &\lesssim \frac{\log(n)}{n}\cdot \tr\big(\Ab(\Ib_n-\eta n^{-1}\Ab)^{2k}\big)\notag\\
&\lesssim \log(n)\cdot \bigg(\frac{ k^*}{(k+1)\eta} + \log(n)\cdot\sum_{i>k^*}\lambda_i\bigg).
\end{align}
Finally, applying Lemma \ref{lemma:concentration_covariance} and setting $\delta = 1/\poly(n)$, we have
\begin{align}\label{eq:concentration_covariance}
\|\Hb-\bSigma\|_2\lesssim \sqrt{\frac{\log(n)}{n}}.
\end{align}
Putting \eqref{eq:concentration_covariance} and \eqref{eq:ei_expA} into \eqref{eq:upperbound_Theta2_temp}, we can obtain 
\begin{align*}
\Theta_2\le \|\eb_i^\top \Xb(\Ib-\eta\bSigma)^k\|_2^2\cdot \|\Hb-\bSigma\|_2\lesssim \frac{\log^{5/2}(n)}{n^{1/2}}\cdot  \bigg(\frac{k^*}{(k+1)\eta} + \sum_{i>k^*}\lambda_i\bigg),
\end{align*}
which completes the proof.
\end{proof}

\subsection{Completing the analysis for fluctuation error: Proof of Theorem \ref{thm:upperbound_fluctuation}} 
Combining the established upper bounds on $\Theta_1$ and $\Theta_2$ in Lemmas \ref{lemma:upperbound_Theta1} and \ref{lemma:upperbound_Theta2} gives the following lemma.
\begin{lemma}\label{lemma:fourth_order_moment}
If the stepsize satisfies $\gamma\le 1/(c\tr(\Hb))$ for some absolute constant $c$, then with probability at least $1-1/\poly(n)$, there exists an absolute constant $C$ such that
\begin{align*}
\cM\circ\cG^k\circ\Hb \preceq C\cdot\bigg[\log(n)\cdot\min\bigg\{\frac{1}{(k+1)\eta}, \|\Hb\|_2\bigg\}\cdot\tr(\Hb) + \frac{ \log^{5/2}(n)}{n^{1/2}}\cdot  \bigg(\frac{k^*}{(k+1)\eta} + \sum_{i>k^*}\lambda_i\bigg)\bigg]\cdot \bSigma.
\end{align*}
\end{lemma}


\begin{lemma}\label{lemma:sgd_opt_error}
For any $t>0$, if the stepsize satisfies $\eta \le 1 / (c \tr(\Hb)\log(t))$ for some absolute constant $c$, then it holds that
\begin{align*}
    \sum_{k=0}^{t-1}  \la  \bSigma, \Eb_{k} \ra &\lesssim  \frac{1}{\eta}\cdot \la \Ib -  (\Ib - \eta \bSigma)^{t} ,  \Eb_0\ra,\notag\\
\sum_{k=0}^{t-1}\frac{\la\bSigma, \Eb_{k}\ra}{t-k}
&\lesssim\frac{1}{\eta t}\big\la(\Ib-(\Ib-\eta\bSigma)^t), \Eb_0\big\ra + \log(t)\big\la (\Ib-\eta\bSigma)^t\bSigma, \Eb_0\big\ra.
\end{align*}
\end{lemma}
\begin{proof}[Proof of Lemma \ref{lemma:sgd_opt_error}]
In this part we seek to bound $\sum_{k=0}^{t-1} \big\la\bSigma, \Eb_{k}\big\ra$ and 
$\sum_{k=0}^{t-1}\frac{\la\bSigma, \Eb_{k}\ra}{t-k}$ in separate.
By \eqref{eq:update_form_Et}, we can get 
\begin{align}\label{eq:upperbound_Sigma_Et}
    \la \bSigma, \Eb_t\ra 
    &\le \la \bSigma, \cG^t\circ \Eb_0\ra + \eta^2 \sum_{k=0}^{t-1}  \la \bSigma,  \cG^{t-1-k}\circ\cM \circ \Eb_{k} \ra \notag\\
    &= \la \cG^t\circ\bSigma,  \Eb_0\ra + \eta^2 \sum_{k=0}^{t-1}  \la  \cM \circ  \cG^{t-1-k}\circ\bSigma, \Eb_{k} \ra\notag \\
    &=  \la (\Ib - \eta \bSigma)^{2t}\bSigma,  \Eb_0\ra + \eta^2 \sum_{k=0}^{t-1}  \la  \cM \circ \big( (\Ib - \eta \bSigma)^{2(t-1-k)}\bSigma \big), \Eb_{k} \ra.
\end{align}
Note that 
\(
(\Ib - \eta \bSigma)^{2(t-1-k)}\bSigma \preceq \frac{1}{\eta (t-k)}\Ib,
\)
and 
\(
\cM \circ \Ib \preceq c \tr(\Hb) \bSigma
\) for some absolute constant $c$,
we then have the following by \eqref{eq:upperbound_Sigma_Et}
\begin{align}
    \la \bSigma, \Eb_t\ra 
    &\le \la (\Ib - \eta \bSigma)^{2t}\bSigma,  \Eb_0\ra + c\eta\tr(\Hb) \sum_{k=0}^{t-1} \frac{ \la  \bSigma, \Eb_{k} \ra}{t-k}.\label{eq:SigmaE-recursion}
\end{align}
We now bound $\sum_{k=0}^{t-1} \la  \bSigma, \Eb_{k} \ra$ by recursively applying \eqref{eq:SigmaE-recursion} to establish
\begin{align*}
    \sum_{k=0}^{t-1}  \la  \bSigma, \Eb_{k} \ra
    &\le  \la \sum_{k=0}^{t-1} (\Ib - \eta \bSigma)^{2k}\bSigma,  \Eb_0\ra + c\eta\tr(\Hb) \sum_{k=0}^{t-1}\sum_{i=0}^{k-1} \frac{ \la  \bSigma, \Eb_{i} \ra}{k-i} \\
    &\le  \frac{1}{\eta} \la \Ib -  (\Ib - \eta \bSigma)^{t} ,  \Eb_0\ra + 2c\eta\tr(\Hb)\log(t) \sum_{i=0}^{t-1}  \la  \bSigma, \Eb_{i} \ra,
\end{align*}
and conclude that
\begin{align}
\sum_{k=0}^{t-1}  \la  \bSigma, \Eb_{k} \ra 
&\le \frac{1}{1-2c\eta\tr(\Hb)\log(t)}\cdot \frac{1}{\eta}\cdot \la \Ib -  (\Ib - \eta \bSigma)^{t} ,  \Eb_0\ra \\
&\le C \cdot \frac{1}{\eta}\cdot \la \Ib -  (\Ib - \eta \bSigma)^{t} ,  \Eb_0\ra. \label{eq:sum_SigmaE_bound}
\end{align}
Similarly, we then bound $\sum_{k=0}^{t-1} \frac{ \la  \bSigma, \Eb_{k} \ra}{t-k}$ by recursively applying \eqref{eq:SigmaE-recursion} to establish
\begin{align*}
    \sum_{k=0}^{t-1} \frac{ \la  \bSigma, \Eb_{k} \ra}{t-k}
    &\le  \la \sum_{k=0}^{t-1} \frac{(\Ib - \eta \bSigma)^{2k}\bSigma}{t-k},  \Eb_0\ra + c\eta\tr(\Hb) \sum_{k=0}^{t-1}\sum_{i=0}^{k-1} \frac{ \la  \bSigma, \Eb_{i} \ra}{(t-k)(k-i)} \\
    &\le  \la \sum_{k=0}^{t-1} \frac{(\Ib - \eta \bSigma)^{2k} \bSigma}{t-k},  \Eb_0\ra + 2c\eta\tr(\Hb)\log(t) \sum_{i=0}^{t-1} \frac{ \la  \bSigma, \Eb_{i} \ra}{t-i},
\end{align*}
so we can conclude that
\begin{align}
\sum_{k=0}^{t-1} \frac{ \la  \bSigma, \Eb_{k} \ra}{t-k} 
& \le \frac{1}{1-2c\eta\tr(\Hb)\log(t)} \la \sum_{k=0}^{t-1} \frac{(\Ib - \eta \bSigma)^{2k} \bSigma}{t-k},  \Eb_0\ra \\
&\lesssim  \sum_{k=0}^{t-1} \frac{(\Ib - \eta \bSigma)^{2k} \bSigma}{t-k},  \Eb_0\ra \\
&\lesssim\bigg(\frac{1}{\eta t}\big\la(\Ib-(\Ib-\eta\bSigma)^t), \Eb_0\big\ra + \log(t)\big\la (\Ib-\eta\bSigma)^t\bSigma, \Eb_0\big\ra\bigg),\label{eq:sum_SigmaEoverK_bound}
\end{align}
where the last inequality is due to 
\[
\sum_{k=0}^{t-1} \frac{(\Ib - \eta \bSigma)^{2k} \bSigma}{t-k} \lesssim \frac{1}{\eta t} (\Ib -  (\Ib - \eta \bSigma)^{t}) + \log(t) \cdot (\Ib - \eta \bSigma)^{t} \bSigma.
\]
\end{proof}

\begin{lemma}\label{lemma:bounds_E0}
For any $t\ge 0$ and $\eta\le 1/(c\tr(\Hb)\log(t))$ for some absolute constant $c$,  it holds that 
\begin{align*}
&\big\la\Ib - (\Ib-\eta\bSigma)^t,\Eb_0\big\ra\le \min\big\{\|\hat\wb\|_2^2, t\eta\cdot\la\bSigma,\Eb_0\ra\big\}\notag\\
&t\eta\cdot\big\la (\Ib-\eta\bSigma)^t\bSigma, \Eb_0\big\ra\le \min\big\{\|\hat\wb\|_2^2, t\eta\cdot\la\bSigma,\Eb_0\ra\big\}
\end{align*}
\end{lemma}
\begin{proof}
According to the definition of $\Eb_t$ and applying zero initialization $\wb_0=\boldsymbol{0}$, then we have $\Eb_0 = \hat\wb\hat\wb^\top\preceq \|\hat\wb\|_2^2\cdot\Ib$. Moreover, note that our choice of stepsize guarantees that $\Ib-\eta\bSigma$ is a PSD matrix, we have
\begin{align*}
\Ib-(\Ib-\eta\bSigma)^t \preceq \Ib, \quad \Ib-(\Ib-\eta\bSigma)^t \preceq t\eta\bSigma,\quad (\Ib-\eta\bSigma)^t\bSigma \preceq \bSigma, \quad (\Ib-\eta\bSigma)^t\bSigma \preceq \frac{1}{t\eta}\cdot \Ib.
\end{align*}
Then it follows that
\begin{align*}
\big\la\Ib - (\Ib-\eta\bSigma)^t,\Eb_0\big\ra &\le \min\big\{\la\Ib,\Eb_0\ra, t\eta\cdot\la\bSigma, \Eb_0\ra\big\}=\min\big\{\|\hat\wb\|_2^2, t\eta\cdot\la\bSigma, \Eb_0\ra\big\}\notag\\
\big\la (\Ib-\eta\bSigma)^t\bSigma, \Eb_0\big\ra&\le \min\bigg\{\la\bSigma, \Eb_0\ra, \frac{1}{t\eta}\cdot\la\Ib, \Eb_0\ra\bigg\} = \min\bigg\{\la\bSigma, \Eb_0\ra, \frac{\|\hat\wb\|_2^2}{t\eta}\bigg\}.
\end{align*}
This completes the proof.

\end{proof}

Now we are ready to complete the proof of Theorem \ref{thm:upperbound_fluctuation}.
\begin{proof}[Proof of Theorem \ref{thm:upperbound_fluctuation}]
By Lemma \ref{lemma:formula_fluctuationErr}, we have
\begin{align*}
\underbrace{\mathrm{FluctuationError}}_{*} \le \frac{\eta^2}{2}\cdot \sum_{k=0}^{t-1}\la\cM\circ\cG^{t-1-k}\circ\Hb, \Eb_k\ra.
\end{align*} 
Additionally, by Lemma \ref{lemma:fourth_order_moment}, we further have
\begin{align*}
(*)&\lesssim\eta^2 \cdot\sum_{k=0}^{t-1} \bigg[\frac{\log(n)}{(t-k)\eta}\cdot\tr(\Hb)+ \frac{ \log^{5/2}(n)}{n^{1/2}}\cdot  \bigg(\frac{k^*}{(t-k)\eta} + \sum_{i>k^*}\lambda_i\bigg)\bigg]\cdot \la\bSigma,\Eb_k\ra\notag\\
&\lesssim \eta \cdot\bigg(\log(n)\tr(\Hb) + \frac{k^*\log^{5/2}(n)}{n^{1/2}}\bigg)\cdot\sum_{k=0}^{t-1} \frac{\la\bSigma, \Eb_k\ra}{t-k} + \eta^2\cdot \frac{\log^{5/2}(n)}{n^{1/2}}\cdot\sum_{i>k^*}\lambda_i\cdot \sum_{k=0}^{t-1}\la\bSigma, \Eb_k\ra.
\end{align*}
Then applying Lemma \ref{lemma:sgd_opt_error}, we can further obtain
\begin{align*}
(*)&\lesssim\eta \cdot\bigg(\log(n)\tr(\Hb) + \frac{k^*\log^{5/2}(n)}{n^{1/2}}\bigg)\cdot\bigg(\frac{1}{\eta t}\big\la(\Ib-(\Ib-\eta\bSigma)^t), \Eb_0\big\ra + \log(t)\big\la (\Ib-\eta\bSigma)^t\bSigma, \Eb_0\big\ra\bigg) \\
&\qquad + \eta\cdot \frac{\log^{5/2}(n)}{n^{1/2}}\cdot\sum_{i>k^*}\lambda_i\cdot  \la \Ib -  (\Ib - \eta \bSigma)^{t} ,  \Eb_0\ra \\
&= \bigg(\frac{\log(n)\tr(\Hb)}{t} + \frac{\log^{5/2}(n)}{n^{1/2} t } \cdot (k^* + \eta t  \sum_{i>k^*} \lambda_i)\bigg)\cdot \big\la(\Ib-(\Ib-\eta\bSigma)^t), \Eb_0\big\ra \big\ra
 \\
&\qquad + \eta\log(t) \cdot\bigg(\log(n)\tr(\Hb) + \frac{k^*\log^{5/2}(n)}{n^{1/2}}\bigg)\cdot \big\la (\Ib-\eta\bSigma)^t\bSigma, \Eb_0\big\ra\\
&\lesssim \bigg[\log(t)\cdot\bigg(\frac{\tr(\Hb)\log(n)}{t}+\frac{k^*\log^{5/2}(n)}{n^{1/2}t}\bigg)+\frac{\log^{5/2}(n)\eta}{n^{1/2}}\cdot\sum_{i>k^*}\lambda_i\bigg)\bigg]\cdot\min\big\{\|\hat\wb\|_2^2, t\eta\cdot\la\bSigma, \Eb_0\ra\big\}. 
\end{align*}
where the last inequality follows from Lemma \ref{lemma:bounds_E0}.

\end{proof}

\section{Risk bounds for Gradient Descent with Early Stopping}
\subsection{Proof of Lemma \ref{lemma:tildeA-bounds}}
\begin{proof}[Proof of Lemma \ref{lemma:tildeA-bounds}]
For the first inequality, note that
\[
\Ib - (\Ib - \eta n^{-1} \Ab)^t \preceq
\begin{cases}
\Ib;\\
n^{-1}\eta t \Ab,
\end{cases}
\]
we then obtain
\begin{align*}
    \tilde{\Ab} := \Ab \big( \Ib - (\Ib - \eta n^{-1}\Ab)^t \big)^{-1}
    \succeq 
    \begin{cases}
    \Ab; \\
    \frac{n}{\eta t} \Ib.
    \end{cases}
\end{align*}
Therefore
\[
\tilde{\Ab} \succeq \frac{1}{2} \big( \Ab + \frac{n}{\eta t} \Ib \big).
\]
For the second inequality, note that 
\begin{align*}
\tilde \Ab - \Ab = \Ab(\Ib-\eta n^{-1}\Ab)^t\big[\Ib - (\Ib-\eta n^{-1}\Ab)^t\big]^{-1}.
\end{align*}
Then it suffices to consider the scalar function $f(x): = nx(1-\eta x)^t/\big[1-(1-\eta x)^t\big]$.
Then we consider two cases: (1) $t\eta x\ge \log (2)$ and (2) $t\eta x< \log (2)$. For the first case, it is clear that
\begin{align*}
\frac{nx(1-\eta x)^t}{1-(1-\eta x)^t}\le \frac{n \cdot 1/(t\eta )}{1-1/2} = \frac{2n}{t\eta},
\end{align*}
where we use the inequality $(1-\eta x)^tx\le 1/(t\eta)$ in the first inequality. For the case of $t\eta x< \log (2)$, we have $(1-\eta x)^t\le 1 - \eta x t/2$ and thus
\begin{align*}
\frac{nx(1-\eta x)^t}{1-(1-\eta x)^t}\le \frac{n x}{\eta x t /2} = \frac{2n}{t\eta}.
\end{align*}
Combining the about results in two cases, we have $f(x)\le 2n/(t\eta)$ and thus
\begin{align*}
\tilde \Ab = \Ab + \Ab(\Ib-\eta n^{-1}\Ab)^t\big[\Ib - (\Ib-\eta n^{-1}\Ab)^t\big]^{-1} \preceq \Ab + \frac{2 n}{t\eta}\cdot \Ib.
\end{align*}
This completes the proof of the second inequality.

\end{proof}


\subsection{Variance Error}
\begin{lemma}\label{lemma:variance_bound}
For any stepsize $\gamma\le c/\tr(\Hb)$ for some absolute constant $c$ and any $k^*\in[d]$, with probability at least $1-1/\poly(n)$, 
\begin{align*}
\EE_{\bepsilon}[\mathrm{VarError}]\lesssim \frac{k^*}{n} + \frac{n}{\big(n / (\eta t) + \sum_{i>k^*}\lambda_i \big)^2 }\cdot \sum_{i>k^*}\lambda_i^2
\end{align*}
\end{lemma}
\begin{proof} By \eqref{eq:bias_var_decomposition_GD}, we have
\begin{align} 
   \EE_{\bepsilon}[\mathrm{VarError}] & := \big\|\Xb^\top \tilde{\Ab}^{-1} \bepsilon\big\|_{\Hb}^2  \lesssim  \tr\big( \Xb \Hb \Xb^\top \tilde{\Ab}^{-2}\big) \lesssim  \tr\Big(\Xb \Hb \Xb^\top \Big(\Ab + \frac{n}{\eta t}\Ib \Big)^{-2}\Big), \label{eq:var-error-ridge}
\end{align}
where the last inequality is by Lemma \ref{lemma:tildeA-bounds}.
One finds that \eqref{eq:var-error-ridge} corresponds to the variance error of ridge regression in \citep{tsigler2020benign} for $\lambda = \frac{n}{\eta t}$. 
Then by Theorem 1 in \citet{tsigler2020benign}, one immediately obtains a bound for GD variance error:
\begin{align*}
    \EE_{\bepsilon}[\mathrm{VarError}] \lesssim\frac{k^*}{n} + \frac{n}{\big(n / (\eta t) + \sum_{i>k^*}\lambda_i \big)^2 }\cdot \sum_{i>k^*}\lambda_i^2,
\end{align*}
where 
\[
k^* := \min\bigg\{ k : n\lambda_{k+1}  \le \frac{n}{\eta t} + \sum_{i>k} \lambda_i \bigg\}.
\]
Setting $\tilde \lambda = n / (\eta t) + \sum_{i>k^*}\lambda_i$ completes the proof.
\end{proof}
\subsection{Bias Error}
\begin{lemma}\label{lemma:bias_bound}
Assume the ground truth $\wb^*$ follows a Gaussian Prior $\wb^*\sim \cN(0, \omega^2\cdot \Ib)$. Then for any stepsize $\gamma\le c/\tr(\Hb)$ for some absolute constant $c$ and any $k^*\in[d]$, with probability at least $1-1/\poly(n)$, 
\begin{align*}
\EE_{\wb^*}[\mathrm{BiasError}]\lesssim \omega^2\cdot\bigg(\frac{\tilde\lambda^2}{n^2}\cdot \sum_{i\le k^*}\frac{1}{\lambda_i} + \sum_{i>k^*}\lambda_i\bigg).
\end{align*}
\end{lemma}
\begin{proof}
Note that given the ground truth $\wb^*$, the bias error is 
\begin{align*}
   \mathrm{BiasError} &:= \norm{ \Hb^{\frac{1}{2}}\big(\Ib - \Xb^\top \tilde{\Ab}^{-1} \Xb  \big) \wb^*}_2^2.
\end{align*}
Further note that 
\[\wb^* \sim \cN(0, \omega^2 \cdot \Ib_d), \]
then taking expectation over $\wb^*$ gives
\begin{align*}
    \EE_{\wb^*}[\mathrm{BiasError}]
    &= \EE_{\wb^*}\big[\norm{ \Hb^{\frac{1}{2}}\big(\Ib - \Xb^\top \tilde{\Ab}^{-1} \Xb  \big) \wb^*}_2^2\big] \\
    &= \omega^2 \cdot \tr\big(\Hb\big(\Ib - \Xb^\top \tilde{\Ab}^{-1} \Xb  \big)^2     \big)\notag\\
    &=\omega^2 \cdot \underbrace{\tr\Big(\Hb\Big(\Ib - \Xb^\top \Big(\Ab+\frac{2n}{t\eta}\Big)^{-2} \Xb  \Big)^2     \Big)}_{*}
\end{align*}
where the last inequality is by Lemma \ref{lemma:tildeA-bounds}. Moreover, note that the quantity $(*)$ is actually the expected bias error of the ridge regression solution with the regularization parameter $2n/(t\eta)$. Therefore, by Theorem 1 in \citet{tsigler2020benign}, we have
\begin{align*}
(*)&\lesssim \EE_{\wb^*\sim \cN(\boldsymbol{0}, \Ib)} \bigg[\bigg(\frac{2n/(\eta t) + \sum_{i>k^*}\lambda_i}{n}\bigg)^2\cdot \|\wb^*_{0:k^*}\|_{\Hb_{0:k^*}^{-1}}^2+\|\wb^*_{k^*:\infty}\|_{\Hb_{k^*:\infty}}^2\bigg]\notag\\
&\eqsim \frac{\tilde\lambda^2}{n^2}\cdot \sum_{i\le k^*}\frac{1}{\lambda_i} + \sum_{i>k^*}\lambda_i,
\end{align*}
where 
\[
k^* := \min\bigg\{ k : n\lambda_{k+1}  \le \frac{n}{\eta t} + \sum_{i>k} \lambda_i \bigg\},
\] and $\tilde \lambda = n/(\eta t) + \sum_{i>k^*}\lambda_i$. 
This completes the proof.
\end{proof}

\subsection{Proof of Theorem \ref{thm:GD_earlystop}}
\begin{proof}[Proof of Theorem \ref{thm:GD_earlystop}]
The proof can be completed by combining Lemmas \ref{lemma:variance_bound} and \ref{lemma:bias_bound}.
\end{proof}

\section{Proof of Corollaries}

\subsection{Proof of Corollary \ref{coro:expected_SGD_error}}
The following lemma will be useful in the proof.
\begin{lemma}\label{lemma:expected_training_error}
Assume $\wb^*\sim\cN(\boldsymbol{0}, \omega^2\cdot\Ib)$ and $\wb_0=\boldsymbol{0}$, then
\begin{align*}
\EE_{\wb^*, \bepsilon}[\la\Eb_0, \bSigma\ra]\lesssim\omega^2\cdot\log(n)\cdot\tr(\Hb) + \sigma^2.
\end{align*}
\end{lemma}
\begin{proof}
Applying the formula of $\hat\wb$ and the initialization $\wb_0 = \boldsymbol{0}$, we have
\begin{align*}
\la\Eb_0,\bSigma\ra = \la \Xb^\top\Ab^{-1}\yb(\Xb^\top\Ab^{-1}\yb)^\top, \bSigma\ra = \frac{1}{n}\|\yb\|_2^2 = \frac{1}{n}\|\Xb\wb^*+\bepsilon\|_2^2\le \frac{2}{n}\|\Xb\wb^*\|_2^2 + \frac{2}{n}\|\bepsilon\|_2^2,
\end{align*}
where the last inequality follows from Young's inequality. Note that $\bepsilon$ is a combination of $n$ independent random variables with variance $\sigma^2$, we have $\EE[\|\bepsilon\|_2^2]=n\sigma^2$. Besides, regarding the first term, we have with probability at least $1-1/\poly(n)$, 
\begin{align*}
\EE[\|\Xb\wb^*\|_2^2] = \omega^2\cdot \tr(\Xb\Xb^\top) \lesssim  \omega^2 \cdot n\cdot \tr(\Hb).
\end{align*}
Combining the above results immediately gives
\begin{align*}
\EE_{\wb^*, \bepsilon}[\la\Eb_0, \bSigma\ra]\lesssim \omega^2\cdot\tr(\Hb) + \sigma^2.
\end{align*}
This completes the proof.

\end{proof}

\begin{proof}[Proof of Corollary \ref{coro:expected_SGD_error}] Plugging Lemma \ref{lemma:expected_training_error} into Theorem \ref{thm:upperbound_fluctuation} and then combining Theorems \ref{thm:upperbound_fluctuation} and \ref{thm:GD_earlystop} completes the proof.

\end{proof}

\subsection{Proof of Corollary \ref{coro:SGD_error_polydecay}}
\begin{proof}[Proof of Corollary \ref{coro:SGD_error_polydecay}]
We will first calculate $k^*$ defined in Corollary \ref{coro:expected_SGD_error}. Note that 
\begin{align*}
k^* = \min\bigg\{k: n\lambda_{k+1} \le \frac{n}{\eta t} + \sum_{i > k} \lambda_i \bigg\},
\end{align*}
and $\sum_{i>k}\lambda_i = \sum_{i>k}i^{-1-r} \eqsim k^{-r}$. Then, it can be shown that
\begin{align}\label{eq:choice_k*}
k^* = (t\eta)^{1/(r+1)}.
\end{align}
Recall that Corollary \ref{coro:expected_SGD_error} shows
\begin{align*}
&\EE_{\sgd, \wb^*}[\mathrm{Risk}(\wb_t)]\notag\\
&\lesssim\omega^2 \cdot
    \underbrace{ \Bigg(  \frac{\tilde{\lambda}^2 }{n^2 }  \cdot \sum_{i\le k^*} \frac{1}{\lambda_i} + \sum_{i>k^*} \lambda_i \Bigg)}_{I_1}
    + \sigma^2\cdot \underbrace{\rbr{\frac{k^*}{n} + \frac{n}{\tilde{\lambda}^2} \sum_{i>k^*}\lambda_i^2 }}_{I_2}+(\omega^2\tr(\Hb) +\sigma^2))\eta\notag\\
    &\quad\cdot \underbrace{\bigg[\log(t)\cdot\bigg(\tr(\Hb)\log(n)+\frac{k^*\log^{5/2}(n)}{n^{1/2}}\bigg)+\frac{\log^{5/2}(n)t\eta}{n^{1/2}}\cdot\sum_{i>k^*}\lambda_i\bigg)\bigg]}_{I_3}.
\end{align*}
Then, applying \eqref{eq:choice_k*} gives
\begin{align*}
&\sum_{i>k^*} \lambda_i\eqsim (k^*)^{-r}\eqsim(t\eta)^{-r/(r+1)},\  \sum_{i>k^*} \lambda_i^2\eqsim (k^*)^{-2r-1}\eqsim(t\eta)^{-(2r+1)/(r+1)},\notag\\
&\sum_{i\le k^*}\frac{1}{\lambda_i} \eqsim (k^*)^{r+2}=(t\eta)^{(r+2)/(r+1)},\ \tilde\lambda\eqsim \frac{n}{t\eta}, \ \tr(\Hb)\eqsim 1
\end{align*}
Putting the above into the formula of $I_1$, $I_2$, and $I_3$, we can get
\begin{align*}
I_1 &\lesssim \frac{n^2/(t\eta)^2}{n^2}\cdot (t\eta)^{(r+2)/(r+1)} + (t\eta)^{-r/(r+1)} \eqsim (t\eta)^{-r/(r+1)};\notag\\
I_2&\lesssim \frac{(t\eta)^{1/(r+1)}}{n} + \frac{n}{n^2/(t\eta)^2}\cdot (t\eta)^{-(2r+1)/(r+1)}\eqsim \frac{(t\eta)^{1/(r+1)}}{n};\notag\\
I_3 &\lesssim\log(t)\cdot\bigg(\log(n)+\frac{(t\eta)^{1/(r+1)}\log^{5/2}(n)}{n^{1/2}}\bigg)+\frac{\log^{5/2}(n)t\eta}{n^{1/2}}\cdot(t\eta)^{-r/(r+1)}\bigg)\\
&  \lesssim\log(t)\cdot\bigg[\log(n)+\frac{\log^{5/2}(n)}{n^{1/2}}\cdot (t\eta)^{1/(r+1)}\bigg].
\end{align*}
Combining the above results leads to
\begin{align*}
\EE_{\sgd, \wb^*}[\mathrm{Risk}(\wb_t)]
&\lesssim\omega^2 \cdot
    (t\eta)^{-r/(r+1)}
    + \sigma^2\cdot \frac{(t\eta)^{1/(r+1)}}{n}\notag\\
    &+(\omega^2 +\sigma^2)\cdot \eta\cdot\log(t)\cdot\bigg[\log(n)+\frac{\log^{5/2}(n)}{n^{1/2}}\cdot (t\eta)^{1/(r+1)}\bigg].
\end{align*}
This completes the proof.


\end{proof}
\bibliographystyle{ims}
\bibliography{refs}

\end{document}